\documentclass{article}
\usepackage{ijcai19}

\usepackage{times}
\usepackage{soul}
\usepackage{url}
\usepackage[hidelinks]{hyperref}
\usepackage[utf8]{inputenc}
\usepackage[small]{caption}
\usepackage{graphicx}
\usepackage{amsmath}
\usepackage{amssymb}
\usepackage{amsthm}
\newtheorem{theorem}{Theorem}
\usepackage{booktabs}
\usepackage{algorithm}
\usepackage{algorithmic}
\usepackage[dvipsnames]{xcolor}
\usepackage{tikz}
\usetikzlibrary{pgfplots.groupplots}
\usetikzlibrary{positioning}
\usetikzlibrary{arrows}

\usepackage{subfigure}

\definecolor{light-gray}{gray}{0.9}
\definecolor{lightblue2}{cmyk}{0.3,0.02,0,0}
\definecolor{lightblue2-text}{cmyk}{0.5,0.2,0,0}
\definecolor{lightblue3}{cmyk}{0.9,0.2,0,0}

\usepackage{pgfplots}
\usetikzlibrary{pgfplots.groupplots}
\usetikzlibrary{positioning}

\title{A Degeneracy Framework for Scalable Graph Autoencoders}

\author{
Guillaume Salha$^{1,2}$\and Romain Hennequin$^1$\and Viet Anh Tran$^1$\And Michalis Vazirgiannis$^2$
\affiliations
$^1$Deezer Research \& Development, Paris, France\\
$^2$École Polytechnique, Palaiseau, France
\emails
research@deezer.com
}

\begin{document}

\maketitle

\begin{abstract}
In this paper, we present a general framework to scale graph autoencoders (AE) and graph variational autoencoders (VAE). This framework leverages graph degeneracy concepts to train models only from a dense subset of nodes instead of using the entire graph. Together with a simple yet effective propagation mechanism, our approach significantly improves scalability and training speed while preserving performance. We evaluate and discuss our method on several variants of existing graph AE and VAE, providing the first application of these models to large graphs with up to millions of nodes and edges. We achieve empirically competitive results w.r.t. several popular scalable node embedding methods, which emphasizes the relevance of pursuing further research towards more scalable graph AE and VAE.
\end{abstract}

\section{Introduction}

Graphs have become ubiquitous in the Machine Learning community, thanks to their ability to efficiently represent the relations among items in various disciplines. Social networks, biological molecules, and communication networks are some of the most famous examples of real-world data usually represented as graphs. Extracting meaningful information from such structures is a challenging task that has initiated considerable research efforts, aiming to solve various problems ranging from link prediction to influence maximization and node clustering.

In particular, over the last decade, there has been an increasing interest in extending and applying Deep Learning methods to graph structures. Gori et al. \cite{gori2005} and Scarselli et al. \cite{scarselli2009} introduced some of the first graph neural network architectures, and were joined by numerous other contributions aiming to generalize CNNs and the convolution operation to graphs, by leveraging spectral graph theory \cite{bruna2014}, its approximations \cite{defferrard2016,kipf2016-1}, or spatial-based approaches \cite{hamilton2017}. Attempts at extending RNNs, GANs, attention mechanisms, or word2vec-like methods to graphs also recently emerged in the literature. For complete references, we refer to Wu et al.~\cite{wu2019comprehensive}'s recent survey on Deep Learning for graphs.

In this paper, we focus on the graph extensions of autoencoders and variational autoencoders. Introduced in the 1980's \cite{Rumelhart1986}, autoencoders (AE) are efficient methods to learn low-dimensional ``encoded'' representations of some input data in an unsupervised way. These models regained significant popularity over the last decade through neural network formulations~\cite{baldi2012autoencoders}. Furthermore, variational autoencoders (VAE) \cite{kingma2013vae}, described as extensions of AE but actually based on quite different mathematical foundations, also recently emerged as a powerful approach for unsupervised learning from complex distributions. They leverage variational inference techniques, and assume that the input data is the observed part of a larger joint model involving some low-dimensional latent variables. We refer to Tschannen et al. \cite{Tschannen2018recentVAE} for a review of the recent advances in VAE-based representation learning.

As illustrated throughout this paper, many efforts have been recently devoted to the generalization of such models to graphs. Graph AE and VAE models appear as effective node embedding methods, i.e., methods learning a low dimensional vector space representation of nodes, with promising applications to various tasks including link prediction, node clustering, matrix completion, and graph generation. However, most existing models suffer from scalability issues, and all experiments are currently limited to graphs with at most a few thousand nodes. The question of how to scale graph AE and VAE models to larger graphs remains open. We propose to address it in this paper. More precisely, our contribution is threefold:
\begin{itemize}
    \item We introduce a general framework to scale graph AE and VAE models, by optimizing the reconstruction loss (for a graph AE) or the variational lower bound objective (for a graph VAE) only from a dense subset of nodes, before propagating the resulting representations in the entire graph. These nodes are selected using graph degeneracy concepts~\cite{malliaros2019}. Such an approach considerably improves scalability while preserving performance.
    \item We apply this framework to large real-world data and discuss empirical results on ten variants of graph AE and VAE models for two learning tasks. To the best of our knowledge, this is the first application of these models to graphs with up to millions of nodes and edges.
    \item We show that these scaled models have competitive performances w.r.t. several popular scalable node embedding methods. This emphasizes the relevance of pursuing further research toward scalable graph autoencoders.
\end{itemize}
This paper is organized as follows. In Section 2, we formally present graph AE and VAE models, as well as their applications and limits. In Section 3, we introduce our ``degeneracy'' framework to improve the scalability of these models. We describe our experimental analysis and discuss extensions of our approach in Section 4, and we~conclude~in~Section~5.

\section{Preliminaries}

In this section, we review some key concepts related to graph AE and VAE models. Throughout this paper, we consider an undirected graph $\mathcal{G} = (\mathcal{V},\mathcal{E})$ with $|\mathcal{V}| = n$ nodes and $|\mathcal{E}| = m$ edges, without self-loops. We denote by $A$ the  $n~\times~n$ adjacency matrix of $\mathcal{G}$. Nodes can possibly have feature vectors of size $d \in \mathbb{N}^*$, stacked up in an $n \times d$ matrix $X$. If $\mathcal{G}$ is featureless, then $X$ is the $n \times n$ identity matrix $I$.

\subsection{Graph Autoencoders}

Over the last few years, several attempts at extending autoencoders to graph structures with \cite{kipf2016-2} or without \cite{wang2016structural} node features have been presented. One of the most popular graph~AE models is the one from Kipf and Welling~\cite{kipf2016-2}, often abbreviated as GAE in the literature. In essence, GAE and other graph AE models aim to learn (in an unsupervised way) a low dimensional latent vector space a.k.a. node embedding space (\textit{encoding}), from which reconstructing the graph structure (\textit{decoding}) should be possible. The $n \times f$ matrix $Z$ of all latent vectors a.k.a. embedding vectors $z_i$ (where $f \in \mathbb{N}^*$ is the dimension of the latent space) is usually the output of a Graph Neural Network (GNN) processing $A$ and, potentially, $X$. To reconstruct $A$ from $Z$, one could resort to another GNN. However, the GAE model and most of its extensions instead rely on a simpler inner product decoder between latent variables, along with a sigmoid activation function $\sigma(\cdot)$ or, if $A$ is weighted, some more complex thresholding. While being simpler, this decoding involves the multiplication of the two dense matrices $Z$ and $Z^T$, which has a quadratic complexity $O(fn^2)$ w.r.t. the number of nodes. To sum up, denoting by $\hat{A}$ the reconstruction of $A$ obtained from the decoder:
$$\hat{A} = \sigma(ZZ^T) \hspace{10pt} \text{with} \hspace{10pt} Z = \text{GNN}(A,X).$$
The GNN encoder usually includes weights at each layer. During the training phase, they are optimized by minimizing a reconstruction loss aiming to compare $\hat{A}$ to $A$, by gradient descent. This loss is often formulated as a Frobenius loss $\|A - \hat{A}\|_F$, where $\|\cdot\|_F$ denotes the Frobenius matrix norm, or alternatively as a weighted cross entropy loss~\cite{kipf2016-2}.

\subsection{Graph Convolutional Networks}

In practice, Kipf and Welling~\cite{kipf2016-2}, and many extensions of their model, assume that this GNN encoder is a Graph Convolutional Network (GCN). Introduced by the same authors in another article~\cite{kipf2016-1}, GCNs leverage both 1) the features $X$, and 2) the graph structure summarized by $A$. In a GCN with $L$ layers, with input layer $H^{(0)} = X$ and output layer $H^{(L)} =Z$, each layer $l \in \{1,\dots,L\}$ computes:
$$H^{(l)} = \text{ReLU}(D^{-1/2}(A + I) D^{-1/2} H^{(l-1)} W^{(l-1)}),$$
i.e., each layer $l$ averages the feature vectors from layer $l-1$ of the neighbors of each node, and combines them with their own feature vectors along with a ReLU activation function: $\text{ReLU}(x) = \max(x,0)$. This ReLU is absent from the output layer $L$. Here, $D$ denotes the diagonal degree matrix of $A + I$, and $D^{-1/2}(A + I) D^{-1/2}$ is therefore the symmetric normalization of $A+I$. $W^{(0)},\dots,W^{(L-1)}$ are the weight matrices to optimize; they can have different dimensions.

Encoding nodes using a GCN is often motivated by complexity reasons. Indeed, the cost of computing each hidden layer evolves linearly w.r.t. $m$ \cite{kipf2016-1}, and the training efficiency of GCNs can also be improved via importance sampling \cite{chen2018fastgcn}. However, recent works \cite{xu2019powerful} highlighted some fundamental limitations of GCN models, which might motivate researchers to consider more powerful albeit more complex GNN encoders (e.g., \cite{bruna2014} computes actual spectral graph convolutions; their model was extended by \cite{defferrard2016}, approximating smooth filters in the spectral domain with Chebyshev polynomials -- GCN being a faster first-order approximation of \cite{defferrard2016}). The scalable degeneracy framework presented in this paper would facilitate the training of such  complex models as encoders within~a~graph~AE~model.

\subsection{Variational Graph Autoencoders}

Kipf and Welling \cite{kipf2016-2} also introduced Variational Graph Autoencoders (VGAE). This VAE model for graphs considers a probabilistic model on the graph structure, involving a latent variable $z_i$ of dimension $f$ for each node $i \in \mathcal{V}$. It will be interpreted as the node's latent vector in the embedding space. More precisely, denoting by $Z$ the $n\times f$ embedding matrix, the inference model (\textit{encoder}) is defined as $q(Z|X,A) = \prod_{i=1}^n q(z_i|X,A)$ where $q(z_i|X,A) = \mathcal{N}(z_i|\mu_i, \text{diag}(\sigma_i^2))$. Parameters of the Gaussian distributions are learned using two two-layer GCNs. Therefore, $\mu$, the matrix of mean vectors $\mu_i$, is defined as $\mu = \text{GCN}_{\mu}(X,A)$. Also, $\log \sigma = \text{GCN}_{\sigma}(X,A)$. Both GCNs share the same weights in the first layer. Finally, a generative model (acting as a \textit{decoder}) aims to reconstruct $A$ using the inner product between latent variables: $p(A|Z) = \prod_{i=1}^n \prod_{j=1}^n p(A_{ij}|z_i, z_j)$ where $p(A_{ij} = 1|z_i, z_j) = \sigma(z_i^Tz_j)$ and $\sigma(\cdot)$  is the sigmoid function. As in Section 2.1, such a reconstruction has a limiting quadratic complexity w.r.t. $n$. Kipf and Welling~\cite{kipf2016-2} optimize the weights of the two GCNs under consideration by maximizing a tractable variational lower bound (ELBO) of the model's likelihood: 
$$\mathcal{L} = \mathbb{E}_{q(Z|X,A)} \Big[\log
p(A|Z)\Big] - \mathcal{D}_{KL}(q(Z|X,A)||p(Z)),$$
where $\mathcal{D}_{KL}(\cdot, \cdot)$ is the Kullback-Leibler divergence. They perform gradient descent, using the \textit{reparameterization trick}~\cite{kingma2013vae} and choosing a Gaussian prior $p(Z) = \prod_i p(z_i) = \prod_i \mathcal{N}(z_i|0,I)$ \cite{kipf2016-2}.

\subsection{Applications, Extensions and Limits}

GAE and VGAE from Kipf and Welling~\cite{kipf2016-2}, as well as other graph AE and VAE models, have been successfully applied to various tasks, such as link prediction \cite{kipf2016-2}, node clustering \cite{wang2017mgae}, and recommendation \cite{berg2018matrixcomp}. Some works also recently tackled multi-task learning problems \cite{tran2018multitask}, added adversarial training schemes enforcing the latent representation to match the prior \cite{pan2018arga}, or proposed RNN-based graph autoencoders to learn graph-level embedding representations \cite{taheri2018rnn}.

We also note the existence of several applications of graph VAE models to biochemical data and small molecular graphs~\cite{molecule2}. Most of them put the emphasis on plausible graph generation using the decoder. Among these models, the ``GraphVAE''~\cite{simonovsky2018graphvae} can reconstruct both 1) the topological graph information, 2) node-level features, and 3) edge-level features. However, it involves a graph matching step with an $O(n^4)$ complexity that prevents the model from scaling, while being acceptable for molecules with tens of nodes.

To the best of our knowledge, all existing experiments on graph AE and VAE models are restricted to small or medium-size graphs, with at most a few thousand nodes and edges. To this date, most models suffer from scalability issues, because they require the training of complex GNN encoders, and/or because they rely on decoders with a quadratic complexity, as Kipf and Welling~\cite{kipf2016-2}. This important problem has already been raised, but without applications to large graphs. For instance, Grover et al.~\cite{grover2018graphite} proposed ``Graphite'', a graph AE replacing the decoder with reverse message passing schemes (presented as more scalable), but only reported results on medium-size graphs (up to $20,000$ nodes). To summarize this section, graph AE and VAE models showed promising results for small and medium-size datasets, but the question of their extension to large graphs~remains~widely~open.

\section{Scaling up Graph AE/VAE with Degeneracy}

This section introduces our degeneracy framework to scale graph AE and VAE models to large graphs. In this section, we assume that nodes are featureless, i.e., that models only learn representations from the graph structure (equivalently, $X = I$). Node features will be re-introduced in~Section~4.

\subsection{Overview of the Framework}

To deal with large graphs, the key idea of our framework is to optimize the reconstruction loss (for a graph AE) or the variational lower bound objective (for a graph VAE) only from a wisely selected subset of nodes, instead of using the entire graph $\mathcal{G}$ which would be intractable. We proceed as follows:
\begin{enumerate}
    \item Firstly, we identify the nodes on which the graph AE/VAE model should be trained, by computing the $k$-core decomposition of the graph under consideration. The selected subgraph is the so-called $k$-degenerate version of the original one~\cite{malliaros2019}. We justify this choice in Subsection 3.2 and explain how to choose the value~of~$k$.
    \item Then, we train the graph AE/VAE model on this $k$-degenerate subgraph. Hence, we derive latent representation vectors (embedding vectors) for the nodes included in this subgraph, but not for the others.
    \item Regarding the nodes of $\mathcal{G}$ that are not in this subgraph, we infer their latent representations using a simple and fast propagation heuristic, presented in Subsection 3.3.
\end{enumerate}

In a nutshell, training the autoencoder (step 2) would still have a high complexity, but now the input graph would be significantly smaller, which would make the training process tractable. Moreover, we will show that steps 1 and 3 have linear running times w.r.t. the number of edges $m$. Therefore,  as we will experimentally verify in Section 4, our strategy significantly improves speed and scalability, and can effectively process large graphs with millions of nodes and edges.

\subsection{Graph Degeneracy}

In this subsection, we detail the first step of our framework, i.e., the identification of a representative subgraph on which the AE or VAE model should be trained. Our method resorts to the $k$-core decomposition~\cite{batagelj2003,malliaros2019}, a powerful tool to analyze the structure of a graph. Formally, the $k$-core, or $k$-degenerate version of a graph $\mathcal{G}$, is the largest subgraph of $\mathcal{G}$ for which every node has a degree larger or equal to $k$ within the subgraph. Therefore, in a $k$-core, each node is connected to at least $k$ nodes, that are themselves connected to at least $k$ nodes. Moreover, the degeneracy number $\delta^*(\mathcal{G})$ of a graph is the maximum value of $k$ for which the $k$-core is not empty. Nodes from each core $k$, denoted $\mathcal{C}_k \subseteq \mathcal{V}$, form a nested chain, i.e., $\mathcal{C}_{\delta^*(\mathcal{G})} \subseteq \mathcal{C}_{\delta^*(\mathcal{G})-1} \subseteq ... \subseteq \mathcal{C}_{0} = \mathcal{V}$. Figure 1 illustrates an example of a $k$-core decomposition. 
\begin{figure}[t]
  \centering
  \scalebox{0.4}{\begin{tikzpicture}[state/.style={circle, draw, minimum size=2cm}]
    \node[fill=MidnightBlue, text=white, shape=circle,draw=black, minimum size=0.7cm] (A) at (0,0) {\textbf{A}} ;
    \node[fill=MidnightBlue, text=white, shape=circle,draw=black, minimum size=0.7cm] (B) at (0,-3) {};
    \node[fill=MidnightBlue, text=white, shape=circle,draw=black, minimum size=0.7cm] (C) at (1.5,-1.5) {};
    \node[fill=MidnightBlue, text=white, shape=circle,draw=black, minimum size=0.7cm] (D) at (3,-1.5) {\textbf{B}};
    \node[fill=MidnightBlue, text=white, shape=circle,draw=black, minimum size=0.7cm] (E) at (4.5,-1.5) {};
    \node[fill=MidnightBlue, text=white, shape=circle,draw=black, minimum size=0.7cm] (F) at (6,0) {\textbf{C}} ;
\node[fill=MidnightBlue, text=white, shape=circle,draw=black, minimum size=0.7cm] (G) at (6,-3) {} ;
\node[fill=lightblue3, text=white, shape=circle,draw=black, minimum size=0.7cm] (H) at (1.5,1.5) {\textbf{D}} ;
\node[fill=lightblue3, text=white, shape=circle,draw=black, minimum size=0.7cm] (I) at (4.5,1.5) {\textbf{E}} ;
\node[fill=lightblue2, shape=circle,draw=black, minimum size=0.7cm] (J) at (0,3) {} ;
\node[fill=lightblue2, shape=circle,draw=black, minimum size=0.7cm] (K) at (6,3) {\textbf{F}} ;
\node[fill=lightblue2, shape=circle,draw=black, minimum size=0.7cm] (L) at (7.5,3) {\textbf{G}};
\node[fill=lightblue3, shape=circle,draw=black, minimum size=0.7cm] (M) at (7.5,-0.5) {};
\node[fill=lightblue3, shape=circle,draw=black, minimum size=0.7cm] (N) at (7.5,-3) {};
\node[fill=lightblue2, shape=circle,draw=black, minimum size=0.7cm] (O) at (9,-0.5) {};
\node[fill=lightblue2, shape=circle,draw=black, minimum size=0.7cm] (P) at (-1.5,0) {};
\node[fill=lightblue2, shape=circle,draw=black, minimum size=0.7cm] (Q) at (-1.5,-3) {};
\node[fill=lightblue2, shape=circle,draw=black, minimum size=0.7cm] (R) at (-2.5,-1.5) {};
\node[fill=lightblue2, shape=circle,draw=black, minimum size=0.7cm] (S) at (9,3) {\textbf{H}};
\node[fill=lightblue2, shape=circle,draw=black, minimum size=0.7cm] (T) at (-1.5,3) {};
\node[fill=light-gray, shape=circle,draw=black, minimum size=0.7cm] (U) at (-3.5,1.5) {};
\node[fill=light-gray, shape=circle,draw=black, minimum size=0.7cm] (V) at (-3.5,-4) {};
\node[fill=light-gray, shape=circle,draw=black, minimum size=0.7cm] (W) at (10,-5) {};
    \path [] (A) edge node {} (B);
    \path [](B) edge node[left] {} (C);
    \path [](A) edge node[left] {} (C);
    \path [](B) edge node[left] {} (D);
    \path [](A) edge node[left] {} (D);
    \path [](C) edge node[left] {} (D);
    \path [](D) edge node[left] {} (E);
    \path [](D) edge node[left] {} (F);
    \path [](D) edge node[left] {} (G);
    \path [](E) edge node[left] {} (F);
     \path [](E) edge node[left] {} (G);
 \path [](F) edge node[left] {} (G);
 \path [](A) edge node[left] {} (H); 
\path [](D) edge node[left] {} (I);
 \path [](F) edge node[left] {} (I);
 \path [](J) edge node[left] {} (H);
 \path [](H) edge node[left] {} (I);
 \path [](I) edge node[left] {} (K);
\path [](K) edge node[left] {} (L);
\path [](F) edge node[left] {} (M);
\path [](G) edge node[left] {} (N);
\path [](M) edge node[left] {} (N);
\path [](M) edge node[left] {} (O);
\path [](P) edge node[left] {} (A);
\path [](Q) edge node[left] {} (B);
\path [](R) edge node[left] {} (P);
\path [](L) edge node[left] {} (S);
\path [](T) edge node[left] {} (J);
\draw [gray,thick,dashed] (-0.5,0.5) -- (6.75,0.5) -- (6.75,-3.5) -- (-0.5,-3.5) -- (-0.5,0.5);
\draw [gray,thick,dashed] (-1,2) -- (8.25,2) -- (8.25,-4.5) -- (-1,-4.5) -- (-1,2);
\draw [gray,thick,dashed] (-3,3.5) -- (9.5,3.5) -- (9.5,-5.5) -- (-3,-5.5) -- (-3,3.5);
\draw [gray,thick,dashed] (-4,4) -- (10.5,4) -- (10.5,-6.5) -- (-4,-6.5) -- (-4,4);
    \node at (3,-3.25) {\Large{\color{MidnightBlue}{\textbf{3-core}}}};
    \node at (3,-4.25) {\Large{\color{lightblue3}{\textbf{2-core}}}};
    \node at (3,-5.25) {\Large{\color{lightblue2-text}{\textbf{1-core}}}};
    \node at (3,-6.25) {\Large{\color{gray}{\textbf{0-core}}}};
\end{tikzpicture}}
  \caption{A graph $\mathcal{G}$ of degeneracy $3$ and its cores. Some nodes are labeled for the purpose of Section 3.3.}
\end{figure}
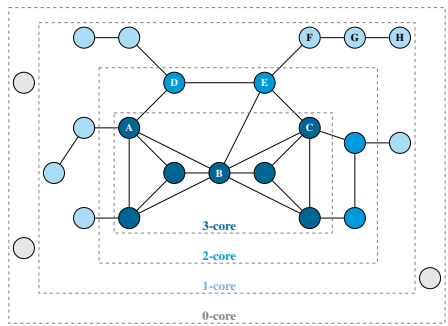

\begin{algorithm}[tb]
\caption{$k$-core Decomposition}
\textbf{Input}: Graph $\mathcal{G} = (\mathcal{V},\mathcal{E})$\\
\textbf{Output}: Set of $k$-cores $\mathcal{C}=\{\mathcal{C}_0,\mathcal{C}_1,...,\mathcal{C}_{\delta^*(\mathcal{G})}\}$

\begin{algorithmic}[1] %[1] enables line numbers
\STATE Initialize $\mathcal{C}=\{\mathcal{V}\}$ and $k = \min_{v \in \mathcal{V}} d(v)$
\FOR{$i = 1$ \textbf{to} $n$} 
\STATE $v =$ node with smallest degree in $\mathcal{G}$
\IF{$d(v) > k$}
\STATE Append $\mathcal{V}$ to $\mathcal{C}$
\STATE $k = d(v)$
\ENDIF
\STATE $\mathcal{V} = \mathcal{V} \setminus \{v\}$ and remove edges linked to $v$
\ENDFOR
\end{algorithmic}
\end{algorithm}

In step 2, we therefore train a graph AE or VAE model, either only on the $\delta^*(\mathcal{G})$-core version of $\mathcal{G}$, or on a larger $k$-core subgraph, i.e., for some $k < \delta^*(\mathcal{G})$. Our justification for this strategy is twofold. The first reason is computational: the $k$-core decomposition can be computed in a linear running time for an undirected graph \cite{batagelj2003}. More precisely, to construct a $k$-core, the strategy is to recursively remove all nodes with degrees lower than $k$ and their edges from $\mathcal{G}$ until no node can be removed, as described in Algorithm 1. It involves sorting nodes by degrees in $O(n)$ time using a variant of bin-sort, and going through all nodes and edges once (see \cite{batagelj2003} for details). The time complexity~is~$O(\max(m,n))$ with $\max(m,n) = m$ in most real-world graphs, and with the same space complexity with sparse matrices. Our second reason to rely on the $k$-core decomposition is that, despite being simple, it has been proven to be a useful tool to extract representative subgraphs over the past years, including for node clustering \cite{giatsidis2014}, keyword extraction in graph-of-words~\cite{tixier2016}, and graph similarity via core-based kernels \cite{nikolentzos2018}. We refer to Malliaros~et~al.~\cite{malliaros2019} for an exhaustive overview of the history, theory, and applications of the $k$-core decomposition.

 \paragraph{On the Selection of $k$.} When selecting $k$, one will usually face a performance/speed trade-off (i.e., choosing small graphs would decrease running times, but also performances), as illustrated in Section 4. Besides, on large graphs ($n~>~50,000$), training a graph AE/VAE on the lowest cores would usually be impossible, due to overly large memory requirements. In our experiments, we will adopt a simple strategy when dealing with such large graphs. We will train our graph AE/VAE models on the largest possible subgraphs. In practice, these subgraphs will be significantly smaller than the original ones ($<5\%$ of nodes). Moreover, when running experiments on medium-size graphs where all cores would be tractable, we will plainly avoid choosing $k < 2$ (since $\mathcal{V} = \mathcal{C}_0 = \mathcal{C}_1$, or $\mathcal{C}_0 \approx \mathcal{C}_1$, in all our graphs). Setting $k=2$, i.e., removing \textit{leaves} from the graph, will empirically appear as a good option, preserving performances w.r.t. models trained on $\mathcal{G}$ while significantly reducing running times by pruning~up~to~50\%~of~nodes.

\subsection{Propagation of Embedding Vectors}

From steps 1 and 2, we compute embedding vectors $z_i$ of dimension $f$ for each node $i$ of the $k$-core. Step 3 aims to infer such representations for the remaining nodes of $\mathcal{G}$, in a scalable way. Nodes are assumed to be featureless, so the only information to leverage comes from the graph structure. Our strategy starts by assigning embedding vectors to nodes directly connected to the $k$-core. We average the values of their embedded neighbors \textit{and} of the nodes being embedded at the same step of the process. For instance, in the graph of Figure 1, to compute $z_D$ and $z_E$ we would solve the system $z_D = \frac{1}{2}(z_A + z_E)$ and $z_E = \frac{1}{3}(z_B + z_C + z_D)$ (or a weighted average, if edges are weighted). Then, we repeat this process on the neighbors of these newly embedded nodes, and so on until no new node is reachable. It is important to consider that nodes $D$ and $E$ are themselves connected. Indeed, node $A$ from the maximal core is also a second-order neighbor of $E$; exploiting this proximity when computing $z_E$ empirically improves performance, as it also strongly impacts all the following nodes whose embedding vectors will then be derived from $z_E$ (in Figure 1, nodes $F, G$ and $H$). 

More generally, let $\mathcal{V}_1$ denote the set of nodes whose embedding vectors are computed, $\mathcal{V}_2$ the set of nodes connected to $\mathcal{V}_1$ and without embedding vectors, $A_1$ the $|\mathcal{V}_1| \times |\mathcal{V}_2|$ adjacency matrix linking $\mathcal{V}_1$ and $\mathcal{V}_2$'s nodes, and $A_2$ the $|\mathcal{V}_2| \times |\mathcal{V}_2|$ adjacency matrix of $\mathcal{V}_2$'s nodes. We normalize $A_1$ and $A_2$ by the total degree in $\mathcal{V}_1 \cup \mathcal{V}_2$, i.e., we divide rows by row sums of the $(A^T_1|A_2)$ matrix row-concatenating $A^T_1$ and $A_2$. We denote by $\tilde{A}_1$ and $\tilde{A}_2$ these normalized matrices. We already learned the $|\mathcal{V}_1|\times f$ embedding  matrix $Z_1$ for nodes in $\mathcal{V}_1$. To implement our strategy, we want to derive a $|\mathcal{V}_2|\times f$ embedding matrix $Z_2$ for nodes in $\mathcal{V}_2$, verifying
$Z_2 = \tilde{A}_1 Z_1 + \tilde{A}_2 Z_2.$ The solution of this system is $Z^* = (I - \tilde{A}_2)^{-1} \tilde{A}_1 Z_1$, which exists since $(I - \tilde{A}_2)$ is strictly diagonally dominant and therefore invertible from the Levy-Desplanques theorem. Unfortunately, the exact computation of $Z^*$ has a cubic complexity. We approximate it by randomly initializing $Z_2$ with values in $[-1,1]$ and iterating $Z_2 = \tilde{A}_1 Z_1 + \tilde{A}_2 Z_2$ until convergence to a fixed point, which is guaranteed to happen exponentially fast as stated below.

\begin{algorithm}[tb]
\caption{Propagation of Embedding Representations}
\textbf{Input}: Graph $\mathcal{G}$, set of embedded nodes $\mathcal{V}_1$, $|\mathcal{V}_1|\times f$ embedding matrix $Z_1$ (already learned), number of iterations $t$\\
\textbf{Output}: An embedding vector for each node of $\mathcal{G}$

\begin{algorithmic}[1] %[1] enables line numbers
\STATE $\mathcal{V}_2=$ set of not-embedded nodes reachable from $\mathcal{V}_1$
\WHILE{$|\mathcal{V}_2| > 0$} 
\STATE $A_1$ = $|\mathcal{V}_1| \times |\mathcal{V}_2|$ adj. matrix linking $\mathcal{V}_1$ and $\mathcal{V}_2$ nodes
\STATE $A_2$ = $|\mathcal{V}_2| \times |\mathcal{V}_2|$ adj. matrix of $\mathcal{V}_2$ nodes \STATE $\tilde{A}_1, \tilde{A}_2 =$ normalized $A_1$, $A_2$ by row sum of $(A^T_1 | A_2)$\STATE Randomly initialize the $|\mathcal{V}_2|\times f$ matrix $Z_2$ \textit{(rows of $Z_2$ will be embedding vectors of $\mathcal{V}_2$'s nodes)}
\FOR{$i = 1$ \textbf{to} $t$}
\STATE $Z_2 = \tilde{A}_1 Z_1 + \tilde{A}_2 Z_2$
\ENDFOR
\STATE $\mathcal{V}_1 = \mathcal{V}_2$
\STATE $\mathcal{V}_2=$ set of not-embedded nodes reachable from $\mathcal{V}_1$
\ENDWHILE
\STATE Assign random vectors to remaining unreachable nodes
\end{algorithmic}
\end{algorithm}

\begin{table*}[!h]
\centering
\begin{footnotesize}
\begin{tabular}{c|c|cc|cccc|c}
\toprule
\textbf{Model}  & \textbf{Size of input} & \multicolumn{2}{c}{\textbf{Mean Perf. on Test Set (in \%)}} & \multicolumn{5}{c}{\textbf{Mean Running Times (in sec.)}}\\
& \textbf{$k$-core} & \tiny \textbf{AUC} & \tiny \textbf{AP} & \tiny $k$-core dec. & \tiny Model train & \tiny Propagation & \tiny \textbf{Total} & \tiny \textbf{Speed gain} \\ 
\midrule
VGAE on $\mathcal{G}$ & - & $83.02 \pm 0.13$ & $\textbf{87.55} \pm \textbf{0.18}$ & - & $710.54$ & - & $710.54$ & - \\
on 2-core & $9,277 \pm 25$  & $\textbf{83.97} \pm \textbf{0.39}$ & $85.80 \pm 0.49$ & $1.35$ & $159.15$ & $0.31$ & $160.81$ & $\times 4.42$ \\
on 3-core & $5,551 \pm 19$ & $\textbf{83.92} \pm \textbf{0.44}$ & $85.49 \pm 0.71$ & $1.35$ & $60.12$ & $0.34$ & $61.81$ & $\times 11.50$\\
on 4-core & $3,269 \pm 30$ & $82.40 \pm 0.66$ & $83.39 \pm 0.75$ & $1.35$ & $22.14$ & $0.36$ & $23.85$ & $\times 29.79$\\
on 5-core & $1,843 \pm 25$ & $78.31 \pm 1.48$ & $79.21 \pm 1.64$ & $1.35$ & $7.71$ & $0.36$ & $9.42$  & $\times 75.43$\\
... & ... & ... & ... & ... & ... & ... & ... & ... \\
on 8-core & $414 \pm 89$ & $67.27 \pm 1.65$ & $67.65 \pm 2.00$ & $1.35$ & $1.55$ & $0.38$ & $\textbf{3.28}$ & $\times \textbf{216.63}$ \\
on 9-core & $149 \pm 93$ & $61.92 \pm 2.88$ & $63.97 \pm 2.86$ & $1.35$ & $1.14$ & $0.38$ & $\textbf{2.87}$ & $\times \textbf{247.57}$ \\
\midrule
DeepWalk & - &$81.04 \pm 0.45$ & $84.04 \pm 0.51$ & - & $342.25$ & - & $342.25$ &- \\
LINE & - & $81.21 \pm 0.31$ & $84.60 \pm 0.37$ & - & $63.52$ & - & $63.52$ &- \\
node2vec & - & $81.25 \pm 0.26$ & $85.55 \pm 0.26$ & - & $48.91$ & - & $48.91$ &- \\
Spectral & - & $83.14 \pm 0.42$ & $86.55 \pm 0.41$ & - & $31.71$ & - & $31.71$ & -\\
\bottomrule
\end{tabular}
\caption{Link prediction on the Pubmed graph ($n=19,717$, $m =44,338$), using the VGAE model, its $k$-core variants, and baselines.}
\end{footnotesize}
\end{table*}

\begin{theorem}
  Let $Z^{(t)}$ denote the $|\mathcal{V}_2|\times f$ matrix obtained by iterating $Z^{(t)} = \tilde{A}_1 Z_1 + \tilde{A}_2 Z^{(t-1)}$ $t$ times starting from $Z^{(0)}$, and let $\|\cdot\|_F$ be the Frobenius norm. Then,~exponentially~fast, 
  $$\|Z^{(t)} -  Z^* \|_F \xrightarrow[{t \rightarrow +\infty}]{} 0.$$ 
\end{theorem}

\begin{proof}
We have $Z^{(t)} - Z^*  = [\tilde{A}_1 Z_1 + \tilde{A}_2 Z^{(t-1)}] - [ \tilde{A}_2 Z^* + (I - \tilde{A}_2) Z^*] = \tilde{A}_1 Z_1 + \tilde{A}_2 Z^{(t-1)} - \tilde{A}_2 Z^* - (I-\tilde{A}_2)(I-\tilde{A}_2)^{-1} \tilde{A}_1 Z_1 = \tilde{A}_2 (Z^{(t - 1)} - Z^*)$. So, $Z^{(t)} -  Z^* =  \tilde{A}^{t}_2 (Z^{(0)} -  Z^*)$. Then, from the~Cauchy-Schwarz~inequality, we have:
$$\|Z^{(t)} -  Z^*\|_F = \|\tilde{A}^{t}_2 (Z^{(0)} -  Z^*)\|_F \leq \|\tilde{A}^{t}_2 \|_F  \|Z^{(0)} -  Z^*\|_F.$$
Furthermore, $\tilde{A}^{t}_2 = P D^{t} P^{-1}$ , with $\tilde{A}_2 = P D P^{-1}$ the eigendecomposition of $\tilde{A}_2$. For the diagonal matrix $D^{t}$ we have $\|D^{t}\|_F = \sqrt{\sum_{i=1}^{|\mathcal{V}_2|} |\lambda^t_i|^{2}} \leq \sqrt{|\mathcal{V}_2|} (\max_i |\lambda_i|)^t$ with $\lambda_i$ the $i$-th eigenvalue of $\tilde{A}_2$. Since $\tilde{A}_2$ has non-negative entries, we derive from the Perron–Frobenius theorem~\cite{lovasz2007} that a)~the maximum absolute value among all eigenvalues of $\tilde{A}_2$ is reached by a nonnegative real eigenvalue, and b)~that $\max_i \lambda_i$ is bounded above by the maximum degree in $\tilde{A}_2$'s graph. By definition, each node in $\mathcal{V}_2$ has at least one connection to $\mathcal{V}_1$; moreover, rows of $\tilde{A}_2$ are normalized by row sums of $(A^T_1|A_2)$, so the maximum degree in $\tilde{A}_2$'s graph is strictly lower than 1. We conclude from a) and b) that $0 \leq |\lambda_i| < 1$ for all $i \in \{1,...,|\mathcal{V}_2|\}$, so $0 \leq \max_i |\lambda_i| < 1$. This result implies that $\|D^{t}\|_F \rightarrow_{t} 0$ exponentially fast, and so does $\|\tilde{A}^{t}_2 \|_F \leq \|P\|_F \|D^{t}\|_F \|P^{-1}\|_F$, then $\|Z^{(t)} - Z^*\|_F$.
 \end{proof}

Algorithm 2 summarizes our propagation process. If some nodes are unreachable by such a process because $\mathcal{G}$ is not connected, then we assign them random vectors. Using sparse representations for $\tilde{A}_1$ and $\tilde{A}_2$, the memory requirement is $O(m +nf)$, and the computational complexity of each evaluation of line 8 also increases linearly w.r.t. the number of edges $m$ in the graph. Moreover, in practice, $t$ is small: we set $t=10$ in our experiments (we illustrate the impact of $t$ in Annex 2). The number of iterations in the while loop of line 2 corresponds to the size of the longest shortest path connecting a node to the $k$-core, a number bounded above by the diameter of the graph, which increases at a $O(\log(n))$ speed in numerous real-world graphs \cite{chakrabarti2006}. In the next section, we will empirically check our claim that steps 1 and 3 run linearly and therefore scale to large graphs with millions of nodes.

\section{Empirical Analysis}
In this section, we empirically evaluate our framework. While all main results are presented here, we report additional and more complete tables in the~supplementary~material.

\subsection{Experimental Setting}

\paragraph{Datasets.} We provide experiments on the three medium-size graphs used by Kipf~and~Welling~\cite{kipf2016-2}: Cora ($n = 2,708$ and $m = 5,429$), Citeseer ($n = 3,327$ and $m = 4,732$) and Pubmed ($n=19,717$ and $m =44,338$); and on two large graphs from Stanford's SNAP project: the Google web graph ($n = 875,713$ and $m=4,322,051$) and the US Patent citation network ($n = 2,745,762$ and $m = 13,965,410$). Details, statistics, and $k$-core decompositions of these graphs are reported in Annex 1. Nodes from Cora, Citeseer, and Pubmed have bag-of-words feature vectors. All graphs are unweighted and, as in \cite{kipf2016-2}, we ignore the~edges'~directions.

\paragraph{Tasks.}

We consider two learning tasks. The first one, as in \cite{kipf2016-2}, is a \textit{link prediction} task. We train models on incomplete versions of graphs where some edges were randomly removed. We create some validation and test sets from the removed edges and the same number of randomly sampled pairs of unconnected nodes. We evaluate the model's ability to classify edges (i.e., the true $A_{ij} = 1$) from non-edges ($A_{ij} = 0$), using the reconstructed value $\hat{A}_{ij} = \sigma(z^T_i z_j)$. The validation and test sets gather $5\%$ and $10\%$ of edges (respectively $2\%$ and $3\%$) for medium-size (resp. large-size) graphs. The (incomplete) training adjacency matrix is used when running Algorithm 2. The validation set is only used to tune hyperparameters. We compare model performances using the \textit{Area Under the Receiver Operating Characteristic (ROC) Curve} (AUC) and \textit{Average Precision} (AP) scores. The second task consists in \textit{clustering nodes} from embedding representations $z_i$. We run $k$-means in embedding spaces, and compare clusters to some ground-truth communities (see Subsection~4.2) using normalized \textit{Mutual~Information}~(MI)~scores.

\paragraph{Models.}
We apply our degeneracy framework to ten graph autoencoders: the seminal GAE and VGAE models~\cite{kipf2016-2} with two-layer GCNs; two deeper variants of GAE and VGAE with 3-layer GCNs; Graphite and Variational Graphite \cite{grover2018graphite}; Pan et al.~\cite{pan2018arga}'s adversarially regularized models (denoted ARGA and ARVGA); and ChebAE/ChebVAE, two variants of GAE/VGAE with ChebNets~\cite{defferrard2016} of order 3 instead of GCNs. We train all models during $200$ epochs. Models return 16-dimensional embedding vectors (32 for Patent). We also compare to the DeepWalk \cite{perozzi2014deepwalk}, LINE \cite{tang2015line}, and node2vec~\cite{grover2016node2vec} node embedding methods, that are explicitly presented as scalable methods. We tune hyperparameters from AUC scores obtained on the validation set (see Annex 2). We also consider a spectral decomposition baseline (embedding axes are the first eigenvectors of $\mathcal{G}$'s Laplacian matrix) and, for node clustering, the scalable ``Louvain'' method~\cite{blondel2008louvain}. We use Python, especially the Tensorflow library, training models on an NVIDIA GTX 1080 GPU and running other operations on a double~Intel~Xeon~Gold~6134~CPU. 
\subsection{Results}

\begin{table}
\centering
\begin{footnotesize}
\begin{tabular}{c|cc|c}  
\toprule
\textbf{Model} & \multicolumn{2}{c}{\textbf{Perf. on Test Set (in \%)}} & \textbf{Total}\\
\tiny (using framework, k=17) & \tiny \textbf{AUC} & \tiny \textbf{AP} &  \textbf{run. time} \\ 
\midrule
GAE   & $94.02 \pm 0.20$ & $ 94.31\pm 0.21$ &  $23 \text{min}$ \\
VGAE   & $93.22 \pm 0.40$ & $93.20 \pm 0.45$ & $\textbf{22 \text{min}}$  \\
DeepGAE   & $93.74 \pm 0.17$ & $92.94 \pm 0.33$ & $24 \text{min}$ \\
DeepVGAE   & $93.12 \pm 0.29$ & $92.71 \pm 0.29$ & $24 \text{min}$ \\
Graphite   & $93.29 \pm 0.33$ & $93.11\pm 0.42$ & $23 \text{min}$ \\
Var-Graphite   & $93.13 \pm 0.35$ & $92.90 \pm 0.39$ &$\textbf{22 \text{min}}$ \\
ARGA  & $93.82 \pm 0.17$ & $94.17 \pm 0.18$ & $23\text{min}$ \\
ARVGA  & $93.00 \pm 0.17$ & $93.38 \pm 0.19$ & $23\text{min}$ \\
ChebGAE   & $\textbf{95.24} \pm \textbf{0.26}$ & $\textbf{96.94} \pm \textbf{0.27}$ & $41\text{min}$ \\
ChebVGAE   & $95.03 \pm 0.25$ & $96.58 \pm 0.21$ & $40\text{min}$ \\
\midrule
node2vec on $\mathcal{G}$& $94.89 \pm 0.63$ & $96.82 \pm 0.72$ & $4\text{h}06$ \\
\textit{(best baseline)} & & & \\
\bottomrule
\end{tabular}
\caption{Link prediction on the Google graph ($n = 875K$, $m=4.3M$) using our framework on the $17$-core ($|\mathcal{C}_{17}|=23,787~\pm~208$) on all graph AE/VAE variants.}
\end{footnotesize}
\end{table}

\paragraph{Medium-Size Graphs.} For Cora, Citeseer, and Pubmed, we apply our framework to all possible subgraphs from the $2$-core to the $\delta^*(\mathcal{G})$-core. We also train models on entire graphs, which is still tractable. Table 1 reports mean AUC and AP scores and their standard errors over 100 runs (training graphs and masked edges are different at each run) along with mean running times, for the \textit{link prediction} task on Pubmed using a VGAE. Sizes of $k$-cores vary over runs due to the edge masking process in \textit{link prediction}; this phenomenon does not occur for \textit{node clustering}. Overall, our framework significantly improves running times w.r.t. training the VGAE on $\mathcal{G}$. Running times decrease when $k$ increases (up to $\times 247.57$ speed gain in Table 1), which was expected since the $k$-core subgraph becomes smaller. We observe this improvement on all other datasets, on both tasks, and for all graph AE/VAE variants (see Annexes 2 and 3). Also, for the low cores, especially the 2-core, performances are consistently competitive w.r.t. models trained on entire graphs, and sometimes better both for \textit{link prediction} (e.g., $+0.95$ point in AUC for 2-core in Table 1) and \textit{node clustering}. This highlights the relevance of our propagation process, and the fact that training models on smaller graphs is easier. Choosing higher cores leads to even faster running times at the price of~a~loss~in~performance.

\paragraph{Large Graphs.} Table 2 summarizes \textit{link prediction} results on Google, for all graph AE/VAE variants trained on the $17$-core using our framework. Also, in Table 3, we summarize \textit{node clustering} results on Patent. Ground-truth clusters are six roughly balanced patent categories. We report performances for all graph AE/VAE variants trained on the $15$-core. Core numbers were selected according to the tractability criterion of Section 3. We average scores over 10 runs. Overall, we reach similar conclusions w.r.t. medium-size graphs, both in terms of good performance and of scalability. However, comparing our results to models trained on $\mathcal{G}$, i.e., without using our framework, is impossible on these large graphs due to overly large memory requirements. We, therefore, compare results to those obtained on several other cores (see Annexes 2 and 3 for complete tables), illustrating once again the inherent performance/speed trade-off when choosing $k$ and validating previous insights.

\paragraph{Graph AE/VAE Variants.} For both tasks, we note that models leveraging adversarial training techniques (ARGA/ARVGA), Graphite's decoder, or a ChebNet-based encoder, tend to perform slightly better than others (e.g., a top $95.24\%$ AUC score for ChebGAE in Table 2). This indicates the relevance of these approaches on our two tasks. 

\paragraph{Baselines.} Additionally, our framework is competitive w.r.t. (non-AE/VAE) baselines. We are significantly faster on large graphs while achieving comparable or outperforming performances in most experiments, which emphasizes the interest of scaling graph AE/VAE models. Furthermore, we specify that 64 dimensions were needed to reach stable results on baselines, against 16 for autoencoders. This suggests that graph AE/VAE models are superior when encoding information in a (very) low-dimensional embedding space. On the other hand, some baselines, e.g., Louvain and node2vec, cluster nodes from Cora and Pubmed more effectively ($+10$ points in MI for Louvain on Cora), which questions the global ability of existing graph AE/VAE~models~to~identify~clusters.

\begin{table}[!t]
\centering
\begin{footnotesize}
\begin{tabular}{c|c|c}
\toprule
\textbf{Model}  & \textbf{Performance (in \%)} & \textbf{Total} \\
\tiny (using framework, k=15)&  \tiny \textbf{Normalized MI} & \textbf{run. time} \\ 
\midrule
GAE & $23.76 \pm 2.25$ & $56\text{min}$ \\
VGAE & $24.53 \pm 1.51$ & $\textbf{54\text{min}}$ \\
DeepGAE & $24.27 \pm 1.10$ & $1\text{h}01$ \\
DeepVGAE & $24.54 \pm 1.23$ & $58\text{min}$ \\
Graphite & $24.22 \pm 1.45$  & $59\text{min}$ \\
Var-Graphite & $24.25 \pm 1.51$ & $58\text{min}$ \\
ARGA & $24.26 \pm 1.18$  & $1\text{h}01$ \\
ARVGA & $24.76 \pm 1.32$ & $58\text{min}$ \\
ChebGAE & $25.23 \pm 1.21$ & $1\text{h}41$ \\
ChebVGAE &  $\textbf{25.30} \pm \textbf{1.22}$ & $1\text{h}38$ \\
\midrule
node2vec on $\mathcal{G}$ & $24.10 \pm 1.64$ & $7\text{h}15$ \\
\textit{(best baseline)}  & & \\
\bottomrule
\end{tabular}
\caption{Node clustering on the Patent graph ($n=2.7M$, $m=13.9M$) using our framework on the $15$-core ($|\mathcal{C}_{15}|=35,432$) on all graph AE/VAE variants.}
\end{footnotesize}
\end{table}

\paragraph{Extensions and Openings.} Based on this last finding, our future research will investigate alternative prior distributions for graph VAE models, aiming to detect clusters/communities in graphs. Moreover, while this paper mainly considered featureless nodes, our method easily extends to attributed graphs, since we can add node features from the $k$-core subgraph as input to graph AE/VAE models. To support this claim, we report experiments on graphs \textit{with node features} for both tasks in Annexes~2~and~3, significantly improving scores (e.g., from $85.24\%$ to $88.10\%$ AUC for the 2-core GAE on Cora). However, node features are currently ignored by our propagation process: future works might, therefore, aim to study more efficient feature integration techniques. Lastly, we aim to obtain theoretical guarantees on $k$-core approximations and extend our~approach~to~directed~graphs.

\section{Conclusion}

In this paper, we introduced a framework based on graph degeneracy to easily scale graph (variational) autoencoders. We provided experimental evidence of its ability to process large graphs effectively. Our work confirms the representational power of these models and identifies several directions that, in future research, should lead towards their improvement. 

\bibliographystyle{plain}
\bibliography{ijcai19}

\clearpage

\section*{Supplementary Material}

\maketitle

This supplementary material provides additional details and more complete tables related to the experimental part of the \textit{A Degeneracy Framework for Scalable Graph Autoencoders} paper. More precisely:
\begin{itemize}
    \item Annex 1 describes our five datasets and their respective $k$-core decompositions.
    \item Annex 2 reports complete tables and experimental settings for the \textit{link prediction} task.
    \item Annex 3 reports complete tables and experimental settings for the \textit{node clustering} task.
\end{itemize}

\subsection*{Annex 1 - Datasets}

\subsubsection*{Medium-size graphs}

For comparison purposes, we ran experiments on the three medium-size graphs used in \cite{kipf2016-2}, i.e., the Cora ($n = 2,708$ and $m = 5,429$), Citeseer ($n = 3,327$ and $m = 4,732$), and Pubmed ($n=19,717$ and $m =44,338$) citation networks. In these graphs, nodes are documents and edges are citation links. As \cite{kipf2016-2}, we ignored edges' directions, i.e., we considered undirected versions of these graphs. Documents/nodes have sparse bag-of-words feature vectors, of sizes $3,703$, $1,433$ and $500$, respectively. Each document also has a class label corresponding to its topic : in Cora (resp. Citeseer, Pubmed), nodes are clustered in 6 classes (resp. 7 classes, 3 classes) that we used as ground-truth communities for the \textit{node clustering} task. Classes are roughly balanced. Data were collected from Kipf's repository\footnote{\href{https://github.com/tkipf/gae}{https://github.com/tkipf/gae}} \cite{kipf2016-2}.

\subsubsection*{Large graphs}
We also provided experiments on two publicly available large graphs from Stanford's SNAP website. The first one is the Google web graph\footnote{\href{http://snap.stanford.edu/data/web-Google.html}{http://snap.stanford.edu/data/web-Google.html}} ($n = 875,713$ and $m=4,322,051$), whose nodes are web pages and directed edges represent hyperlinks between these pages. Data do not include ground-truth communities. The second one is the US Patent citation network\footnote{\href{http://snap.stanford.edu/data/cit-Patents.html}{http://snap.stanford.edu/data/cit-Patents.html}} ($n = 2,745,762$ and $m = 13,965,410$), originally released by the National Bureau of Economic Research (NBER) and representing citations between patents. Nodes have classes corresponding to 6 patent categories; we removed nodes without classes from the Patent original graph. For both graphs, we once again ignored edges' directions.

\subsubsection*{$k$-core decompositions}
Tables 4 to 8 detail the $k$-core decomposition of each graph. We used the Python implementation provided in the networkx library. For Citeseer, the $1$-core is smaller than the $0$-core because this graph includes isolated nodes. Figures 2 to 6 illustrate the evolution of the number of nodes in the $k$-core induced by increasing $k$. We note that, for Google (resp. for Patent), it was intractable to train autoencoders on the 0 to 15-cores (resp. on 0 to 13-cores) due to memory errors. Therefore, in our experiments, we trained models on the 16 to 20-cores (resp. 14 to 18-cores).

\begin{table}[H]
\centering
\begin{footnotesize}
\begin{tabular}{c|c|c}  
\toprule
\textbf{$k$} &  \textbf{Number of nodes} &  \textbf{Number of edges} \\
& \textbf{in $k$-core} & \textbf{in $k$-core} \\
\midrule
$0$ & $2,708$ & $5,278$ \\ 
$1$ & $2,708$ & $5,278$ \\ 
$2$ & $2,136$  & $4,768$ \\
$3$ & $1,257$ & $3,198$ \\
$4$ ($\delta^*(\mathcal{G})$) & $174$ & $482$ \\
\bottomrule
\end{tabular}
\caption{$k$-core decomposition of the Cora graph.}
\end{footnotesize}
\end{table}

\begin{figure}[H]
  \centering
  \scalebox{0.45}{\includegraphics{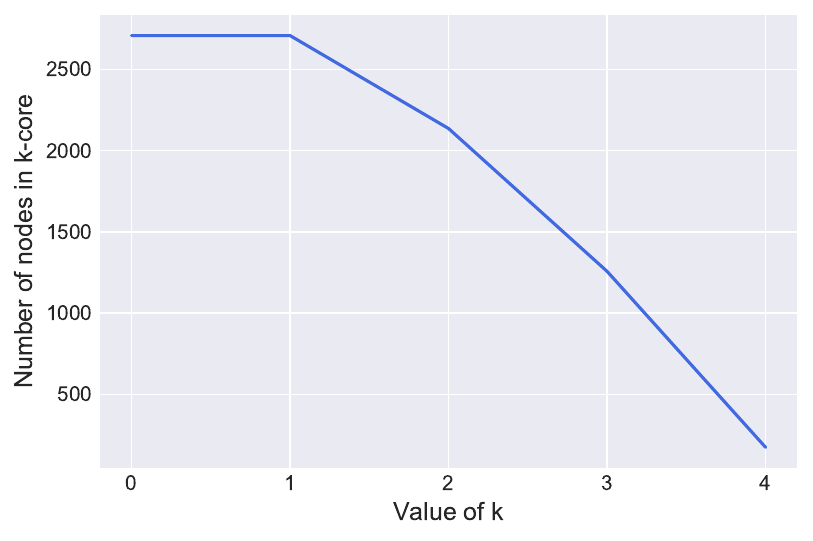}}
  \caption{$k$-core decomposition of the Cora graph.}
\end{figure}

\begin{table}[H]
\centering
\begin{footnotesize}
\begin{tabular}{c|c|c}  
\toprule
\textbf{$k$} &  \textbf{Number of nodes} &  \textbf{Number of edges} \\
& \textbf{in $k$-core} & \textbf{in $k$-core} \\
\midrule
$0$ & $3,327$ & $4,552$ \\
$1$ & $3,279$ & $4,552$ \\
$2$ & $1,601$ & $3,213$ \\
$3$ & $564$ & $1,587$ \\
$4$ & $203$ & $765$ \\
$5$ & $70$ & $319$ \\
$6$ & $28$ & $132$ \\
$7$ ($\delta^*(\mathcal{G})$)& $18$ & $86$ \\
\bottomrule
\end{tabular}
\caption{$k$-core decomposition of the Citeseer graph.}
\end{footnotesize}
\end{table}

\begin{figure}[H]
  \centering
  \scalebox{0.45}{\includegraphics{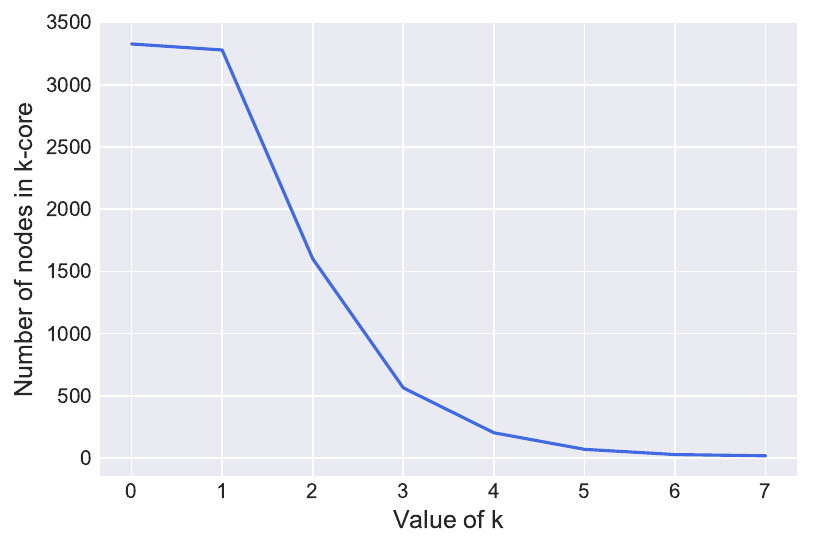}}
  \caption{$k$-core decomposition of the Citeseer graph.}
\end{figure}

\begin{table}[H]
\centering
\begin{footnotesize}
\begin{tabular}{c|c|c}  
\toprule
\textbf{$k$} &  \textbf{Number of nodes} &  \textbf{Number of edges} \\
& \textbf{in $k$-core} & \textbf{in $k$-core} \\
\midrule
$0$ & $19,717$ & $44,324$ \\
$1$ & $19,717$ & $44,324$ \\
$2$ & $10,404$ & $35,011$ \\
$3$ & $6,468$ & $27,439$ \\
$4$ & $4,201$ & $21,040$ \\
$5$ & $2,630$ & $15,309$ \\
$6$ & $1,569$ & $10,486$ \\
$7$ & $937$ & $7,021$ \\
$8$ & $690$ & $5,429$ \\
$9$ & $460$ & $3,686$ \\
$10$ ($\delta^*(\mathcal{G})$) & $137$ & $1,104$ \\
\bottomrule
\end{tabular}
\caption{$k$-core decomposition of the Pubmed graph.}
\end{footnotesize}
\end{table}

\begin{table}[H]
\centering
\begin{footnotesize}
\begin{tabular}{c|c|c}  
\toprule
\textbf{$k$} &  \textbf{Number of nodes} &  \textbf{Number of edges} \\
& \textbf{in $k$-core} & \textbf{in $k$-core} \\
\midrule
$0$ & $875,713 $ & $4,322,051 $ \\
$1$ & $875,713 $ & $4,322,051 $ \\
$2$ & $711,870 $ & $4,160,100$ \\
$3$ & $581,712 $ & $3,915,291$ \\
$4$ & $492,655 $ & $3,668,104$ \\
$5$ & $424,155 $ & $3,416,251$ \\
$6$ & $367,361 $ & $3,158,776$ \\
$7$ & $319,194 $ & $2,902,138$ \\
... & ... & ... \\
$16$ & $53,459 $ & $676,076$ \\
$17$ & $40,488 $ & $519,077$ \\
$18$ & $29,554 $ & $384,478$ \\
$19$ & $19,989 $ & $263,990$ \\
$20$ & $11,073 $ & $154,000$ \\
... & ... & ... \\
$43$ & $103 $ & $2,513$ \\
$44$ ($\delta^*(\mathcal{G})$) & $48 $ & $1,121$ \\
\bottomrule
\end{tabular}
\caption{$k$-core decomposition of the Google graph.}
\end{footnotesize}
\end{table}

\begin{table}[H]
\centering
\begin{footnotesize}
\begin{tabular}{c|c|c}  
\toprule
\textbf{$k$} &  \textbf{Number of nodes} &  \textbf{Number of edges} \\
& \textbf{in $k$-core} & \textbf{in $k$-core} \\
\midrule
$0$ & $2,745,762 $ & $13,965,410$ \\
$1$ & $2,745,762 $ & $13,965,410$ \\
$2$ & $2,539,676 $ & $13,762,533$ \\
$3$ & $2,299,008 $ & $13,295,322$ \\
$4$ & $2,011,518 $ & $12,468,383$ \\
$5$ & $1,671,474 $ & $11,179,712$ \\
$6$ & $1,284,078 $ & $9,363,176$ \\
$7$ & $888,944 $ & $7,164,181$ \\
... & ... & ... \\
$14$ & $46,685 $ & $717,597$ \\
$15$ & $35,432 $ & $576,755$ \\
$16$ & $28,153 $ & $480,436$ \\
$17$ & $22,455 $ & $400,463$ \\
... & ... & ... \\
$63$ & $109 $ & $4,232$ \\
$64$ ($\delta^*(\mathcal{G})$)& $106 $ & $4,043$ \\
\bottomrule
\end{tabular}
\caption{$k$-core decomposition of the Patent graph.}
\end{footnotesize}
\end{table}

\begin{figure}[H]
  \centering
  \scalebox{0.45}{\includegraphics{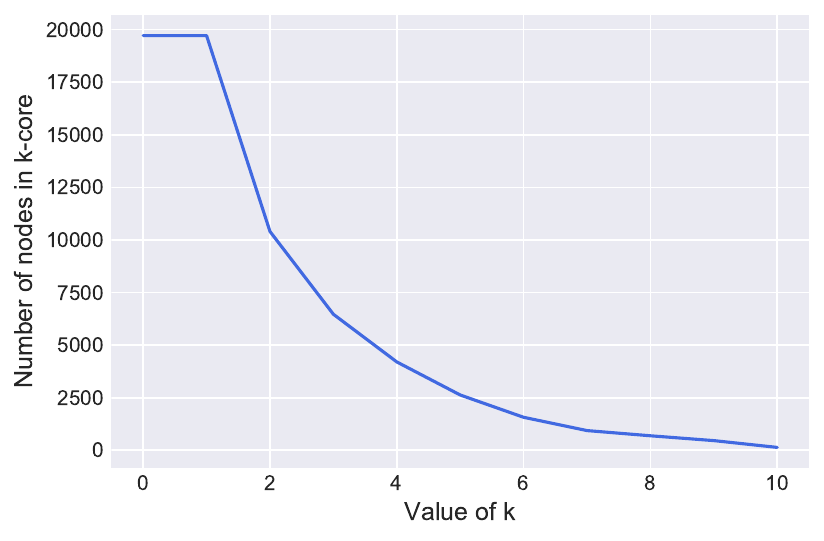}}
  \caption{$k$-core decomposition of the Pubmed graph.}
\end{figure}

\begin{figure}[H]
  \centering
  \scalebox{0.45}{\includegraphics{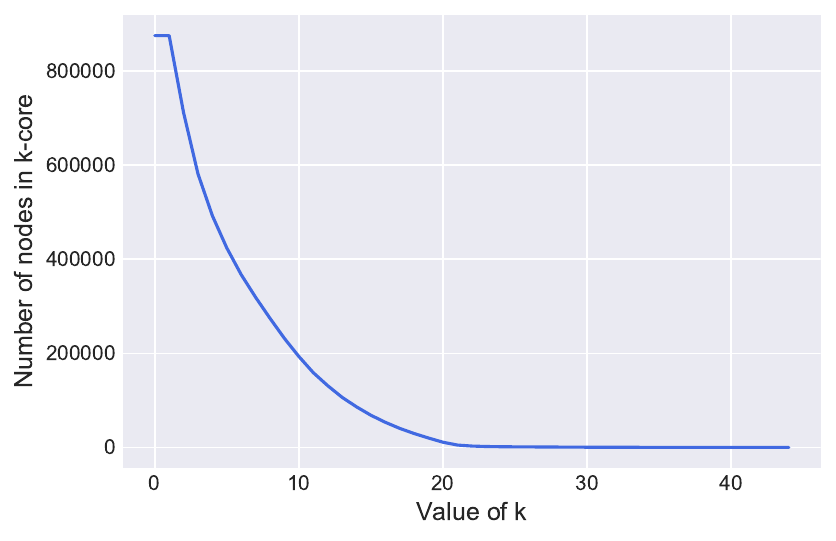}}
  \caption{$k$-core decomposition of the Google graph.}
\end{figure}

\begin{figure}[H]
  \centering
  \scalebox{0.45}{\includegraphics{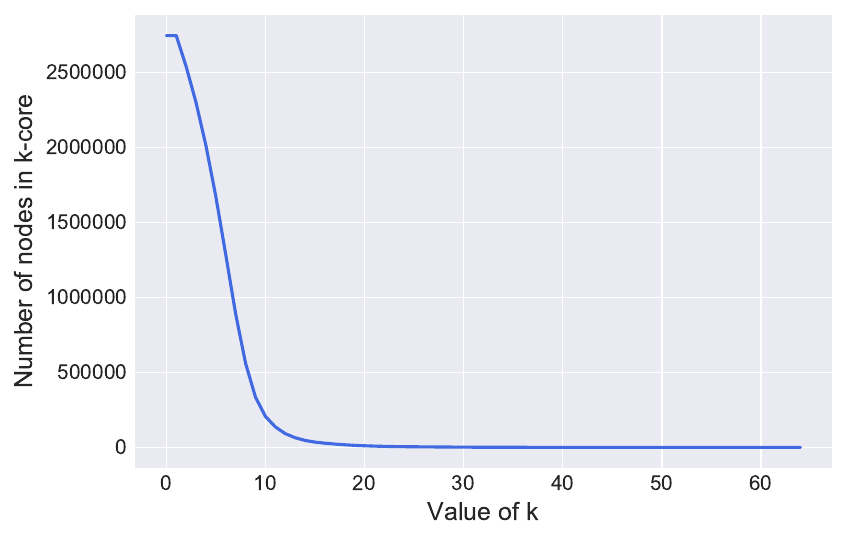}}
  \caption{$k$-core decomposition of the Patent graph.}
\end{figure}

\subsection*{Annex 2 - Link Prediction}

\begin{figure*}[ht!]
  \centering
  \subfigure[Cora: GAE trained on $3$-core]{
  \scalebox{0.4}{\includegraphics{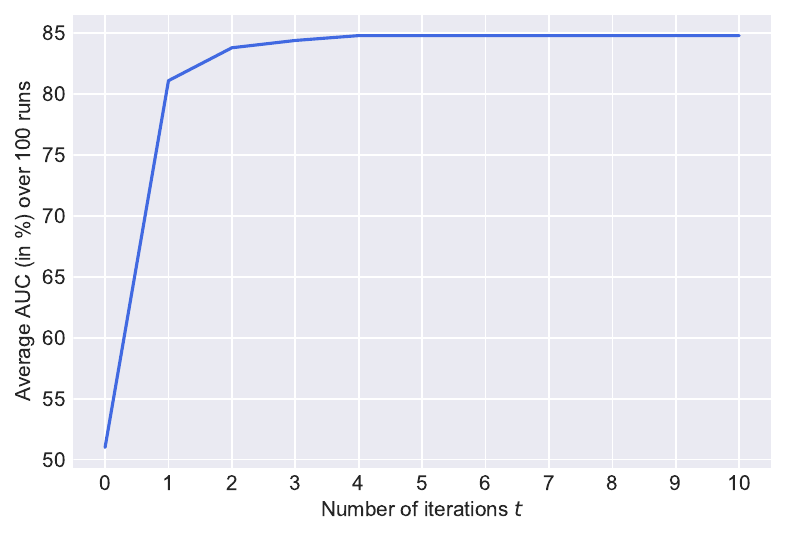}}}
  \subfigure[Pubmed: VGAE trained on $8$-core]{
  \scalebox{0.4}{\includegraphics{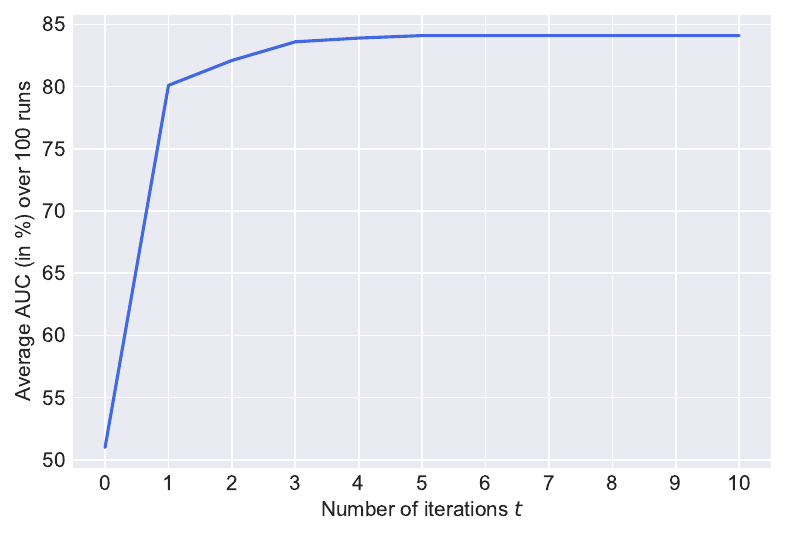}}}
  \subfigure[Google: VGAE trained on $18$-core]{
  \scalebox{0.4}{\includegraphics{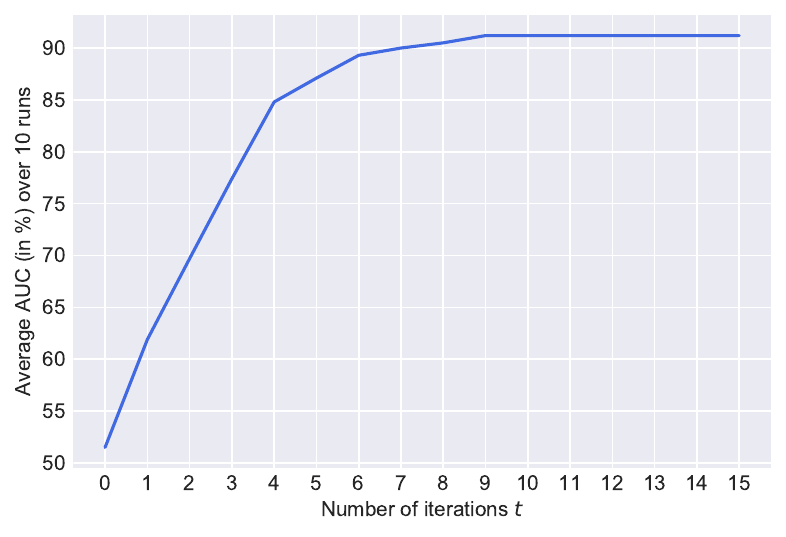}}}
  \caption{Impact of the number of iterations $t$ during progagation on mean AUC scores.}
\end{figure*}

In this Annex 2, we provide complete tables for the \textit{link prediction} task. The main paper discusses the conclusions from experiments; here, we focus on completeness and details regarding our experimental setting.

\subsubsection*{Medium-size graphs}

We applied our framework to all possible subgraphs from the $2$-core to the $\delta^*(\mathcal{G})$-core. We also trained graph AE/VAE models on entire graphs for comparison. Tables 9 to 11 report mean AUC and AP scores and their standard errors over 100 runs (training graphs and masked edges are different for each run) along with mean running times, for the VGAE model. Sizes of $k$-cores vary over runs due to the edge masking process. We obtained comparable performance/speed trade-offs for graph AE/VAE variants: for brevity, we, therefore, only report results on the 2-core~for~these~models.

\subsubsection*{Large graphs}

For the Google and Patent graphs, comparison with models trained on the entire $\mathcal{G}$ was impossible, due to overly large memory requirements. As a consequence, we applied our framework to the five largest $k$-cores (in terms of nodes) that were tractable using our machines. Tables 12 and 13 report mean AUC/AP scores and standard errors over 10 runs (training graphs and masked edges are different for each run) along with mean running times, for the VGAE model. For other models, we only report results on the second-largest cores for brevity. We chose the second-largest cores (17-core for Google, 15-core for Patent) instead of the largest cores (16-core for Google, 14-core for Patent)~to~lower~running~times.
\subsubsection*{Graph AE/VAE training}

All graph AE/VAE models were trained during $200$ epochs to return 16-dimensional embeddings, except for Patent ($500$ epochs, 32-dimensional). We included a 32-dimensional hidden layer in GCN encoders (two for DeepGAE and DeepVGAE), used the Adam optimizer, trained models without dropout, and with a learning rate of $0.01$. We performed full-batch gradient descent and used the reparameterization trick \cite{kingma2013vae}. We used the Tensorflow public implementations of these models (see corresponding references). Overall, our setting is quite similar to \cite{kipf2016-2}, and we indeed managed to reproduce their scores when training GAE and VGAE on entire graphs, with however larger standard errors. This difference comes from the fact that we used 100 different train/test splits, while they launched all runs on fixed dataset splits (randomness, therefore, only came from the initialization).

\subsubsection*{Impact of the number of iterations $t$}

We illustrate the impact of the number of iterations $t$ during propagation on performances in Figure 7. We display the evolution of mean AUC scores w.r.t. $t$, for three different graphs. For $t>5$ (resp. $t>10$) in medium-size graphs (resp. large graphs), scores become stable. We specify that the number of iterations has a negligible impact on running times. Therefore, in our experiments, we set $t = 10$ for all models leveraging~our~degeneracy~framework.

\subsubsection*{Baselines}

For DeepWalk \cite{perozzi2014deepwalk}, we trained models from 10 random walks of length 80 per node with a window size of 5, on a single epoch for each graph. We used similar hyperparameters for node2vec \cite{grover2016node2vec}, setting $p = q =1$, and LINE \cite{tang2015line} enforcing second-order proximity. We directly used the public implementations provided by the authors. Due to unstable and underperforming results with 16-dimensional embeddings, we had to increase dimensions up to 64, to compete with autoencoders. For the spectral embedding baseline, we also computed embeddings from 64 Laplacian eigenvectors. 

In our experiments, we noticed some slight differences w.r.t. \cite{kipf2016-2} regarding baselines, which we explain by our modifications in train/test splits and by different hyperparameters. However, these slight variations do not impact the conclusions of our experiments.
We specify that the spectral embedding baseline is not scalable, due to the required eigendecomposition of the Laplacian matrix. Moreover, we chose not to report results for DeepWalk on large graphs due to too large training times ($>20\text{h}$) on our machines. This does not question the scalability of this method (which is quite close to node2vec), but it suggests possible improvements in the existing implementation.

\begin{table*}[h!]
\centering
\begin{footnotesize}
\begin{tabular}{c|c|cc|cccc}
\toprule
\textbf{Model}  & \textbf{Size of input} & \multicolumn{2}{c}{\textbf{Mean Perf. on Test Set (in \%)}} & \multicolumn{4}{c}{\textbf{Mean Running Times (in sec.)}}\\
& \textbf{$k$-core} & \tiny \textbf{AUC} & \tiny \textbf{AP} & \tiny $k$-core dec. & \tiny Model train & \tiny Propagation & \tiny \textbf{Total} \\ 
\midrule
VGAE on $\mathcal{G}$ & - & $84.07 \pm 1.22$ & $\textbf{87.83} \pm \textbf{0.95}$ & - & $15.34$ & - & $15.34$ \\
on 2-core & $1,890 \pm 16$  & $\textbf{85.24} \pm \textbf{1.12}$ & $ 87.37 \pm 1.13$ & $0.16$ & $8.00$ & $0.10$ & $8.26$ \\
on 3-core & $862 \pm 26$ & $84.53 \pm 1.33$ & $85.04 \pm 1.87$ & $0.16$ & $2.82$ & $0.11$ & $3.09$  \\
on 4-core & $45 \pm 13$ & $72.33 \pm 4.67$ & $71.98 \pm 4.97$ & $0.16$ & $\textbf{0.98}$ & $0.12$ & $\textbf{1.26}$ \\
\midrule
GAE on 2-core & $1,890 \pm 16$ &$85.17 \pm 1.02$ & $87.26 \pm 1.12$ & $0.16$ & $8.05$ & $0.10$ & $8.31$ \\
DeepGAE on 2-core & $1,890 \pm 16$ & $86.25 \pm 0.81$ & $87.92 \pm 0.78$ & $0.16$ & $8.24$ & $0.10$ & $8.50$ \\
DeepVGAE on 2-core & $1,890 \pm 16$ & $86.16 \pm 0.95$ & $87.71 \pm 0.98$ & $0.16$ & $8.20$ & $0.10$ & $8.46$ \\
Graphite on 2-core & $1,890 \pm 16$& $86.35 \pm 0.82$ & $88.18 \pm 0.84$ & $0.16$ & $9.41$ & $0.10$ & $9.67$ \\
Var-Graphite on 2-core & $1,890 \pm 16$ & $\textbf{86.39} \pm \textbf{0.84}$ & $88.05 \pm 0.80$ & $0.16$ & $9.35$ & $0.10$ & $9.61$ \\
ARGA on 2-core & $1,890 \pm 16$ & $85.82 \pm 0.88$ & $ 88.22\pm 0.70$ & $0.16$ & $7.99$ & $0.10$ & $8.25$ \\
ARVGA on 2-core & $1,890 \pm 16$ &$85.74 \pm 0.74$ & $88.14 \pm 0.74$ & $0.16$ & $7.98$ & $0.10$ & $8.24$ \\
ChebGAE on 2-core & $1,890 \pm 16$ & $86.15 \pm 0.54$ & $88.01 \pm 0.39$ & $0.16$ & $15.78$ & $0.10$ & $16.04$ \\
ChebVGAE on 2-core & $1,890 \pm 16$ &$86.30 \pm 0.49$ & $\textbf{88.29} \pm \textbf{0.50}$ & $0.16$ & $15.65$ & $0.10$ & $15.91$ \\
\midrule
GAE with node features on 2-core & $1,890 \pm 16$ &$\textbf{88.10} \pm \textbf{0.87}$ & $89.36 \pm 0.88$ & $0.16$ & $8.66$ & $0.10$ & $8.92$ \\
VGAE with node features on 2-core & $1,890 \pm 16$ &$87.97 \pm 0.99$ & $\textbf{89.53} \pm \textbf{0.96}$ & $0.16$ & $8.60$& $0.10$ & $8.86$ \\
\midrule
DeepWalk & - &$83.02 \pm 1.21$ & $84.41 \pm 1.23$ & - & $38.50$ & - & $38.50$ \\
LINE & - & $83.49 \pm 1.31$ & $84.42 \pm 1.39$ & - & $11.55$ & - & $11.55$ \\
node2vec & - & $83.52 \pm 1.47$ & $84.60 \pm 1.23$ & - & $8.42$ & - & $8.42$ \\
Spectral & - & $\textbf{86.53} \pm \textbf{1.02}$ & $\textbf{87.41} \pm \textbf{1.12}$ & - & $\textbf{2.78}$ & - & $\textbf{2.78}$ \\
\bottomrule
\end{tabular}
\caption{Link prediction on Cora ($n = 2,708$, $m = 5,429$), using VGAE on all cores, graph AE/VAE variants on 2-core, and baselines.}
\end{footnotesize}
\end{table*}

% ok

\begin{table*}
\centering
\begin{footnotesize}
\begin{tabular}{c|c|cc|cccc}
\toprule
\textbf{Model}  & \textbf{Size of input} & \multicolumn{2}{c}{\textbf{Mean Perf. on Test Set (in \%)}} & \multicolumn{4}{c}{\textbf{Mean Running Times (in sec.)}}\\
& \textbf{$k$-core} & \tiny \textbf{AUC} & \tiny \textbf{AP} & \tiny $k$-core dec. & \tiny Model train & \tiny Propagation & \tiny \textbf{Total} \\ 
\midrule
VGAE on $\mathcal{G}$ & - & $\textbf{78.10} \pm \textbf{1.52}$ & $\textbf{83.12} \pm \textbf{1.03}$ & - & $22.40$ & - & $22.40$ \\
on 2-core & $1,306 \pm 19$  & $77.50 \pm 1.59$ & $ 81.92 \pm 1.41$ & $0.15$ & $4.72$ & $0.11$ & $4.98$ \\
on 3-core& $340 \pm 13$ & $76.40 \pm 1.72$ & $80.22 \pm 1.42$ & $0.15$ & $1.75$ & $0.14$ & $2.04$  \\
on 4-core & $139 \pm 13$ & $73.34 \pm 2.43$ & $75.49 \pm 2.39$ & $0.15$ & $1.16$ & $0.16$ & $1.47$ \\
on 5-core & $46 \pm 10$ & $65.47 \pm 3.16$ & $68.50 \pm 2.77$ & $0.15$ & $\textbf{0.99}$ & $0.16$ & $\textbf{1.30}$ \\
\midrule
GAE on 2-core & $1,306 \pm 19$ &$78.35 \pm 1.51$ & $82.44 \pm 1.32$ & $0.15$ & $4.78$ & $0.11$ & $5.04$ \\
DeepGAE on 2-core & $1,306 \pm 19$ & $\textbf{79.32} \pm \textbf{1.39}$ & $82.80 \pm 1.33$ & $0.15$ & $4.99$ &  $0.11$ & $5.25$ \\
DeepVGAE on 2-core & $1,306 \pm 19$ & $78.52 \pm 1.02$ & $82.43 \pm 0.97$ & $0.15$ & $4.95$ &  $0.11$ & $5.21$ \\
Graphite on 2-core & $1,306 \pm 19$ & $78.61 \pm 1.58$ & $82.81 \pm 1.24$ & $0.15$ & $5.88$ & $0.11$ & $6.14$ \\
Var-Graphite on 2-core & $1,306 \pm 19$ & $78.51 \pm 1.62$ & $82.72 \pm 1.25$ & $0.15$ & $5.86$ &  $0.11$ & $6.12$ \\
ARGA on 2-core & $1,306 \pm 19$ & $78.89 \pm 1.33$ & $82.89 \pm 1.03$ & $0.15$ & $4.54$ &  $0.11$ & $4.80$ \\
ARVGA on 2-core & $1,306 \pm 19$ &$77.98 \pm 1.39$ & $82.39 \pm 1.09$ & $0.15$ & $4.40$ &  $0.11$ & $4.66$ \\
ChebGAE on 2-core & $1,306 \pm 19$ & $78.62 \pm 0.95$ & $\textbf{83.22} \pm \textbf{0.89}$ & $0.15$ & $8.87$ &  $0.11$ & $9.13$ \\
ChebVGAE on 2-core & $1,306 \pm 19$ &$78.75 \pm 1.03$ & $\textbf{83.23} \pm \textbf{0.76}$ & $0.15$ & $8.75$ &  $0.11$ & $9.01$ \\
\midrule
GAE with node features on 2-core & $1,306 \pm 19$ & $81.21 \pm 1.86$ & $\textbf{83.99} \pm \textbf{1.52}$ & $0.15$ & $5.51$ & $0.11$ & $5.77$ \\
VGAE with node features on 2-core & $1,306 \pm 19$ & $\textbf{81.88} \pm \textbf{2.23}$ & $83.83 \pm 1.85$ & $0.15$ & $5.70$& $0.11$ & $5.96$ \\
\midrule
DeepWalk & - &$76.92 \pm 1.15$ & $79.30 \pm 0.79$ & - & $41.20$ & - & $41.20$ \\
LINE & - & $77.12 \pm 1.09$ & $79.92 \pm 0.80$ & - & $12.41$ & - & $12.41$ \\
node2vec & - & $76.98 \pm 1.27$ & $79.74 \pm 0.84$ & - & $9.98$ & - & $9.98$ \\
Spectral & - & $\textbf{80.56} \pm \textbf{1.41}$ & $\textbf{83.98} \pm \textbf{1.08}$ & - & $\textbf{3.77}$ & - & $\textbf{3.77}$ \\
\bottomrule
\end{tabular}
\caption{Link prediction on Citeseer ($n = 3,327$, $m = 4,732$), using VGAE on all cores*, graph AE/VAE variants on 2-core, and baselines. * 6-core and 7-core are not reported due to their frequent vanishing after edge masking.}
\end{footnotesize}
\end{table*}

% Ok

\begin{table*}
\centering
\begin{footnotesize}
\begin{tabular}{c|c|cc|cccc}
\toprule
\textbf{Model}  & \textbf{Size of input} & \multicolumn{2}{c}{\textbf{Mean Perf. on Test Set (in \%)}} & \multicolumn{4}{c}{\textbf{Mean Running Times (in sec.)}}\\
& \textbf{$k$-core} & \tiny \textbf{AUC} & \tiny \textbf{AP} & \tiny $k$-core dec. & \tiny Model train & \tiny Propagation & \tiny \textbf{Total} \\ 
\midrule
VGAE on $\mathcal{G}$ & - & $83.02 \pm 0.13$ & $\textbf{87.55} \pm \textbf{0.18}$ & - & $710.54$ & - & $710.54$  \\
on 2-core & $9,277 \pm 25$  & $\textbf{83.97} \pm \textbf{0.39}$ & $85.80 \pm 0.49$ & $1.35$ & $159.15$ & $0.31$ & $160.81$ \\
on 3-core & $5,551 \pm 19$ & $\textbf{83.92} \pm \textbf{0.44}$ & $85.49 \pm 0.71$ & $1.35$ & $60.12$ & $0.34$ & $61.81$ \\
on 4-core & $3,269 \pm 30$ & $82.40 \pm 0.66$ & $83.39 \pm 0.75$ & $1.35$ & $22.14$ & $0.36$ & $23.85$ \\
on 5-core & $1,843 \pm 25$ & $78.31 \pm 1.48$ & $79.21 \pm 1.64$ & $1.35$ & $7.71$ & $0.36$ & $9.42$ \\
... & ... & ... & ... & ... & ... & ... & ... \\
on 8-core & $414 \pm 89$ & $67.27 \pm 1.65$ & $67.65 \pm 2.00$ & $1.35$ & $1.55$ & $0.38$ & $3.28$  \\
on 9-core & $149 \pm 93$ & $61.92 \pm 2.88$ & $63.97 \pm 2.86$ & $1.35$ & $\textbf{1.14}$ & $0.38$ & $\textbf{2.87}$ \\
\midrule
GAE on 2-core & $9,277 \pm 25$ &$84.30 \pm 0.27$ & $\textbf{86.11} \pm \textbf{0.43}$ & $1.35$ & $167.25$ &$0.31$ & $168.91$ \\
DeepGAE on 2-core & $9,277 \pm 25$ & $84.61 \pm 0.54$ & $85.18 \pm 0.57$ & $1.35$ & $166.38$ & $0.31$ & $168.04$ \\
DeepVGAE on 2-core & $9,277 \pm 25$ & $84.46 \pm 0.46$ & $85.31 \pm 0.45$ & $1.35$ & $157.43$ & $0.31$ & $159.09$ \\
Graphite on 2-core & $9,277 \pm 25$ & $84.51 \pm 0.58$ & $85.65 \pm 0.58$ & $1.35$ & $167.88$ & $0.31$ & $169.54$ \\
Var-Graphite on 2-core & $9,277 \pm 25$ & $84.30 \pm 0.57$ & $85.57 \pm 0.58$ & $1.35$ & $158.16$ & $0.31$ & $159.82$ \\
ARGA on 2-core & $9,277 \pm 25$ & $84.37 \pm 0.54$ & $86.07 \pm 0.45$ & $1.35$ & $164.06$ & $0.31$ & $165.72$ \\
ARVGA on 2-core & $9,277 \pm 25$ &$84.10 \pm 0.53$ & $85.88 \pm 0.41$ & $1.35$ & $155.83$ & $0.31$ & $157.49$ \\
ChebGAE on 2-core & $9,277 \pm 25$ & $\textbf{84.63} \pm \textbf{0.42}$ & $86.05 \pm 0.70$ & $1.35$ & $330.37$ & $0.31$ & $332.03$ \\
ChebVGAE on 2-core & $9,277 \pm 25$ &$84.54 \pm 0.48$ & $86.00 \pm 0.63$ & $1.35$ & $320.01$ & $0.31$ & $321.67$ \\
\midrule
GAE with node features on 2-core & $9,277 \pm 25$ &$84.94 \pm 0.54$ & $85.83 \pm 0.58$ & $1.35$ & $168.62$ & $0.31$ & $170.28$ \\
VGAE with node features on 2-core & $9,277 \pm 25$ &$\textbf{85.81} \pm \textbf{0.68}$ & $\textbf{88.01} \pm \textbf{0.53}$ & $1.35$ & $164.10$& $0.31$ & $165.76$ \\
\midrule
DeepWalk & - &$81.04 \pm 0.45$ & $84.04 \pm 0.51$ & - & $342.25$ & - & $342.25$ \\
LINE & - & $81.21 \pm 0.31$ & $84.60 \pm 0.37$ & - & $63.52$ & - & $63.52$ \\
node2vec & - & $81.25 \pm 0.26$ & $85.55 \pm 0.26$ & - & $48.91$ & - & $48.91$ \\
Spectral & - & $\textbf{83.14} \pm \textbf{0.42}$ & $\textbf{86.55} \pm \textbf{0.41}$ & - & $\textbf{31.71}$ & - & $\textbf{31.71}$ \\
\bottomrule
\end{tabular}
\caption{Link prediction on Pubmed ($n=19,717$, $m =44,338$), using VGAE on all cores*, graph AE/VAE variants on 2-core, and baselines. * 10-core is not reported due to its frequent vanishing after edge masking.}
\end{footnotesize}
\end{table*}

% ok 

\begin{table*}
\centering
\begin{footnotesize}
\begin{tabular}{c|c|cc|cccc}
\toprule
\textbf{Model}  & \textbf{Size of input} & \multicolumn{2}{c}{\textbf{Mean Perf. on Test Set (in \%)}} & \multicolumn{4}{c}{\textbf{Mean Running Times (in sec.)}}\\
& \textbf{$k$-core} & \tiny \textbf{AUC} & \tiny \textbf{AP} & \tiny $k$-core dec. & \tiny Model train & \tiny Propagation & \tiny \textbf{Total} \\ 
\midrule
VGAE on 16-core & $36,854 \pm 132$  & $\textbf{93.56} \pm \textbf{0.38}$ & $\textbf{93.34} \pm \textbf{0.31}$ & $301.16$ & $2,695.42$ & $25.54$ & $3,022.12~(50\text{min})$ \\
on 17-core  & $23,787 \pm 208$ & $\textbf{93.22} \pm \textbf{0.40}$ & $\textbf{93.20} \pm \textbf{0.45}$ & $301.16$ & $1,001.64$ & $28.16$ & $1,330.86~(22 \text{min})$  \\
on 18-core  & $13,579 \pm 75$ & $91.24 \pm 0.40$ & $92.34 \pm 0.51$ & $301.16$ & $326.76$ & $28.20$ & $656.12~(11\text{min})$ \\
on 19-core & $6,613 \pm 127$ & $87.79 \pm 0.31$ & $89.13 \pm 0.29$ & $301.16$ & $82.19$ & $28.59$ & $411.94~(7 \text{min})$ \\
on 20-core & $3,589 \pm 106$ & $81.74 \pm 1.17$ & $83.51 \pm 1.22$ & $301.16$ & $\textbf{25.59}$ & $28.50$ & $\textbf{355.55~(6 \text{min})}$ \\
\midrule
GAE on $17$-core & $23,787 \pm 208$ &$94.02 \pm 0.20$ & $ 94.31\pm 0.21$ &  $301.16$ & $1,073.18$ & $28.16$ & $1,402.50~(23 \text{min})$ \\
DeepGAE on $17$-core & $23,787 \pm 208$& $93.74 \pm 0.17$ & $92.94 \pm 0.33$ &  $301.16$ & $1,137.24$ & $28.16$ & $1,466.56~(24 \text{min})$ \\
DeepVGAE on $17$-core & $23,787 \pm 208$ & $93.12 \pm 0.29$ & $92.71 \pm 0.29$ &  $301.16$ & $1,088.41$ & $28.16$ & $1,417.73~(24 \text{min})$ \\
Graphite on $17$-core & $23,787 \pm 208$ & $93.29 \pm 0.33$ & $93.11\pm 0.42$ &  $301.16$ & $1,033.21$ & $28.16$ & $1,362.53~(23 \text{min})$ \\
Var-Graphite on $17$-core & $23,787 \pm 208$ & $93.13 \pm 0.35$ & $92.90 \pm 0.39$ &  $301.16$ & $989.90$ & $28.16$ & $1,319.22~(22 \text{min})$ \\
ARGA on $17$-core & $23,787 \pm 208$ & $93.82 \pm 0.17$ & $94.17 \pm 0.18$ &  $301.16$ & $1,053.95$ & $28.16$ & $1,383.27~(23 \text{min})$ \\
ARVGA on $17$-core & $23,787 \pm 208$ &$93.00 \pm 0.17$ & $93.38 \pm 0.19$ &  $301.16$ & $1,027.52$ & $28.16$ & $1,356.84~(23 \text{min})$ \\
ChebGAE on $17$-core & $23,787 \pm 208$ & $\textbf{95.24} \pm \textbf{0.26}$ & $\textbf{96.94} \pm \textbf{0.27}$ &  $301.16$ & $2,120.66$ & $28.16$ & $2,449.98~(41 \text{min})$ \\
ChebVGAE on $17$-core & $23,787 \pm 208$ &$95.03 \pm 0.25$ & $96.82 \pm 0.72$ &  $301.16$ & $2,086.07$ & $28.16$ & $2,415.39~(40 \text{min})$ \\
\midrule
LINE & - & $93.52 \pm 0.43$ & $95.90 \pm 0.59$ & - & $19,699.02$ & - & $19,699.02~(5\text{h}19)$ \\
node2vec & - & $\textbf{94.89} \pm \textbf{0.63}$ & $\textbf{96.82} \pm \textbf{0.72}$ & - & $14,762.78$ & - & $14,762.78~(4\text{h}06)$ \\
\bottomrule
\end{tabular}
\caption{Link prediction on Google ($n = 875,713$, $m=4,322,051$), using VGAE on 16 to 20 cores, graph AE/VAE variants on $17$-core, and baselines.}
\end{footnotesize}
\end{table*}

\begin{table*}
\centering
\begin{footnotesize}
\begin{tabular}{c|c|cc|cccc}
\toprule
\textbf{Model}  & \textbf{Size of input} & \multicolumn{2}{c}{\textbf{Mean Perf. on Test Set (in \%)}} & \multicolumn{4}{c}{\textbf{Mean Running Times (in sec.)}}\\
& \textbf{$k$-core} & \tiny \textbf{AUC} & \tiny \textbf{AP} & \tiny $k$-core dec. & \tiny Model train & \tiny Propagation & \tiny \textbf{Total} \\ 
\midrule
VGAE on $14$-core & $38,408 \pm 147$  & $\textbf{88.48} \pm \textbf{0.35}$ & $\textbf{88.81} \pm \textbf{0.32}$ & $507.08$ & $3,024.31$ & $122.29$ & $3,653.68~(1\text{h}01)$ \\
on $15$-core & $29,191 \pm 243$ & $88.16 \pm 0.50$ & $88.37 \pm 0.57$ & $507.08$  & $1,656.46$ & $123.47$ &$2,287.01~(38\text{min})$ \\
on $16$-core  & $23,132 \pm 48$ & $87.85 \pm 0.47$ & $88.02 \pm 0.48$ & $507.08$ & $948.09$ & $124.26$  &$1,579.43~(26\text{min})$ \\
on $17$-core  & $18,066 \pm 143$ & $87.34 \pm 0.56$ & $87.64 \pm 0.47$ & $507.08$ & $574.25$ & $126.55$  & $1,207.88~(20\text{min})$ \\
on $18$-core  & $13,972 \pm 86$ & $87.27 \pm 0.55$ & $87.78 \pm 0.51$ & $507.08$ & $\textbf{351.73}$ & $127.01$ &$\textbf{985.82~(16\text{min})}$ \\
\midrule
GAE on $15$-core &  $29,191 \pm 243$ & $87.59 \pm 0.29$ & $87.30 \pm 0.28$ &  $507.08$  & $1,880.11$ &  $123.47$ & $2,510.66~(42 \text{min})$ \\
DeepGAE on $15$-core & $29,191 \pm 243$ & $87.71 \pm 0.31$ & $87.64 \pm 0.19$ &  $507.08$  & $2,032.15$ &  $123.47$ & $2,662.70~(44 \text{min})$ \\
DeepVGAE on $15$-core &  $29,191 \pm 243$ & $87.03 \pm 0.54$ & $87.20 \pm 0.44$ &  $507.08$  & $1,927.33$ &  $123.47$ & $2,557.88~(43 \text{min})$ \\
Graphite on $15$-core &  $29,191 \pm 243$ & $85.19 \pm 0.38$ & $86.01 \pm 0.31$ &  $507.08$  & $1,989.72$ &  $123.47$ & $2,620.27~(44 \text{min})$ \\
Var-Graphite on $15$-core &  $29,191 \pm 243$ & $85.37 \pm 0.30$ & $86.07 \pm 0.24$ &  $507.08$  & $1,916.79$ &  $123.47$ & $2,547.34~(42 \text{min})$ \\
ARGA on $15$-core &  $29,191 \pm 243$ & $\textbf{89.22} \pm \textbf{0.10}$ & $\textbf{89.40} \pm \textbf{0.11}$ &  $507.08$  & $2,028.46$ &  $123.47$ & $2,659.01~(44 \text{min})$ \\
ARVGA on $15$-core &  $29,191 \pm 243$ &$87.18 \pm 0.17$ & $87.39 \pm 0.33$ &  $507.08$  & $1,915.53$ &  $123.47$ & $2,546.08~(42 \text{min})$ \\
ChebGAE on $15$-core &  $29,191 \pm 243$  & $88.53 \pm 0.20$ & $88.91 \pm 0.20$ &  $507.08$  & $3,391.01$ & $123.47$ & $4,021.56~(1\text{h}07)$ \\
ChebVGAE on $15$-core & $29,191 \pm 243$ &$88.75 \pm 0.19$ & $89.07 \pm 0.24$ &  $507.08$  & $3,230.52$ & $123.47$ & $3,861.07~(1\text{h}04)$ \\
\midrule
LINE & - & $90.07 \pm 0.41$ & $94.52 \pm 0.49$ & - & $33,063.80$ & - & $33,063.80~(9\text{h}11)$ \\
node2vec & - & $\textbf{95.04} \pm \textbf{0.25}$ & $\textbf{96.01} \pm \textbf{0.19}$ & - & $26,126.01$ & - & $26,126.01~(7\text{h}15)$ \\
\bottomrule
\end{tabular}
\caption{Link prediction on Patent ($n = 2,745,762$, $m = 13,965,410$), using VGAE on 14 to 18 cores, graph AE/VAE variants on $15$-core, and baselines.}
\end{footnotesize}
\end{table*}

\subsection*{Annex 3 - Node Clustering}

Lastly, we provide complete tables for the \textit{node clustering} task. As before, we focus on the completeness of results, and refer to the main paper for interpretations.

In experiments, we evaluated the quality of node clustering from latent representations, running $k$-means in embeddings and reporting normalized \textit{Mutual Information} (MI) scores. We used scikit-learn's implementation with $k$-means++ initialization. We do not report any result for the Google graph, due to the lack of ground-truth communities. Also, we obtained very low scores on the Citeseer graph, with all methods, which suggests that node features are more valuable than the graph structure to explain labels. As a consequence, we also omit this graph and focus on Cora, Pubmed, and Patent in Tables~14~to~16. We constructed tables in a similar fashion w.r.t. Annex~2. 

Models were trained with identical hyperparameters w.r.t. the \textit{link prediction} task. Contrary to Annex 2, we do not report spectral clustering results because graphs are not connected. Nevertheless, we compared to the ``Louvain'' method, a popular scalable algorithm to cluster nodes by maximizing the  modularity \cite{blondel2008louvain}, using the Python implementation provided in the python-louvain's library.

\begin{table*}
\centering
\begin{footnotesize}
\begin{tabular}{c|c|c|cccc}
\toprule
\textbf{Model}  & \textbf{Size of input} & \textbf{Mean Performance (in \%)} & \multicolumn{4}{c}{\textbf{Mean Running Times (in sec.)}}\\
& \textbf{$k$-core} & \tiny \textbf{MI} & \tiny $k$-core dec. & \tiny Model train & \tiny Propagation & \tiny \textbf{Total} \\ 
\midrule
VGAE on $\mathcal{G}$ & - & $29.52 \pm 2.61$  & - & $15.34$ & - & $15.34$ \\
on 2-core & $2,136$  & $ 34.08 \pm 2.55$ & $0.16$ & $9.94$ & $0.10$ & $10.20$ \\
on 3-core& $1,257$ & $\textbf{36.29} \pm \textbf{2.52}$ & $0.16$ & $4.43$ & $0.11$ & $4.70$  \\
on 4-core & $174$ & $35.93 \pm 1.88$ & $0.16$ & $\textbf{1.16}$ & $0.12$ & $\textbf{1.44}$ \\

\midrule
VGAE with node features on $\mathcal{G}$ & - & $\textbf{47.25} \pm \textbf{1.80}$ &  - & $15.89$ & - & $15.89$ \\
on 2-core & $2,136$  & $45.09 \pm 1.91$ & $0.16$ & $10.42$ & $0.10$ & $10.68$ \\
on 3-core & $1,257$ & $40.96 \pm 2.06$ & $0.16$ & $4.75$ & $0.11$ & $5.02$  \\
on 4-core & $174$ & $38.11 \pm 1.23$ & $0.16$ & $\textbf{1.22}$ & $0.12$ & $\textbf{1.50}$ \\
\midrule
GAE on 2-core & $2,136$ &$34.91 \pm 2.51$ &  $0.16$ & $10.02$ & $0.10$ & $10.28$ \\
DeepGAE on 2-core & $2,136$ & $35.30 \pm 2.52$ & $0.16$ & $10.12$ & $0.10$ & $10.38$ \\
DeepVGAE on 2-core & $2,136$ & $34.49 \pm 2.85$ & $0.16$ & $10.09$ & $0.10$ & $10.35$ \\
Graphite on 2-core & $2,136$ & $33.91 \pm 2.17$ & $0.16$ & $10.97$ & $0.10$ & $11.23$ \\
Var-Graphite on 2-core & $2,136$ & $33.89 \pm 2.13$  & $0.16$ & $10.91$ & $0.10$ & $11.17$ \\
ARGA on 2-core & $2,136$ & $34.73 \pm 2.84$  & $0.16$ & $9.99$ & $0.10$ & $10.25$ \\
ARVGA on 2-core & $2,136$ &$33.36 \pm 2.53$  & $0.16$ & $9.97$ & $0.10$ & $10.23$ \\
ChebGAE on 2-core & $2,136$ & $\textbf{36.52} \pm \textbf{2.05}$ & $0.16$ & $19.22$ & $0.10$ & $19.48$ \\
ChebVGAE on 2-core & $2,136$ &$37.83 \pm 2.11$  & $0.16$ & $20.13$ & $0.10$ & $20.39$ \\
\midrule
DeepWalk & - &$40.37 \pm 1.51$  & - & $38.50$ & - & $38.50$ \\
LINE & - & $39.78 \pm 1.24$  & - & $11.55$ & - & $11.55$ \\
Louvain & - & $\textbf{46.76} \pm \textbf{0.82}$ & - & $\textbf{1.83}$ & - & $\textbf{1.83}$ \\
node2vec & - & $43.45 \pm 1.32$ & - & $8.42$ & - & $8.42$ \\
\bottomrule
\end{tabular}
\caption{Node clustering on Cora ($n = 2,708$, $m = 5,429$), using VGAE on all cores, graph AE/VAE variants on 2-core, and baselines.}
\end{footnotesize}
\end{table*}

\begin{table*}
\centering
\begin{footnotesize}
\begin{tabular}{c|c|c|cccc}
\toprule
\textbf{Model}  & \textbf{Size of input} & \textbf{Mean Performance (in \%)} & \multicolumn{4}{c}{\textbf{Mean Running Times (in sec.)}}\\
& \textbf{$k$-core} & \tiny \textbf{MI} & \tiny $k$-core dec. & \tiny Model train & \tiny Propagation & \tiny \textbf{Total} \\ 
\midrule
VGAE on $\mathcal{G}$ & - & $22.36 \pm 0.25$  & - & $707.77$ & - & $707.77$ \\
on 2-core& $10,404$  & $ 23.71 \pm 1.83$ & $1.35$ & $199.07$ & $0.30$ & $200.72$ \\
on 3-core & $6,468$ & $\textbf{25.19} \pm \textbf{1.59}$ & $1.35$ & $79.26$ & $0.34$ & $80.95$  \\
on 4-core & $4,201$ & $24.67 \pm 3.87$ &  $1.35$ & $34.66$ & $0.35$ & $36.36$ \\
on 5-core & $2,630$ & $17.90 \pm 3.76$ &  $1.35$ & $14.55$ & $0.36$ & $16.26$ \\
... & ... & ... & ... & ... & ... & ... \\
on 10-core & $137$ & $10.79 \pm 1.16$ &  $1.35$ & $\textbf{1.15}$ & $0.38$ & $\textbf{2.88}$ \\
\midrule
VGAE with node features on $\mathcal{G}$ & - & $\textbf{26.05} \pm \textbf{1.40}$ &  - & $708.59$ & - & $708.59$ \\
on 2-core & $10,404$  & $24.25 \pm 1.92$ & $1.35$ & $202.37$ & $0.30$ & $204.02$ \\
on 3-core & $6,468$ & $23.26 \pm 3.42$ & $1.35$ & $82.89$ & $0.34$ & $84.58$  \\
on 4-core & $4,201$ & $20.17 \pm 1.73$ & $1.35$ & $36.89$ & $0.35$ & $38.59$ \\
on 5-core & $2,630$ & $18.15 \pm 2.04$ & $1.35$ & $16.08$ & $0.36$ & $17.79$ \\
... & ... & ... & ... & ... & ... & ... \\
on 10-core & $137$ & $11.67 \pm 0.71$ & $1.35$ & $\textbf{0.97}$ & $0.38$ & $\textbf{2.70}$ \\
\midrule
GAE on 2-core & $10,404$ &$22.76 \pm 2.25$ &  $1.35$ & $203.56$ & $0.30$ & $205.21$ \\
DeepGAE on 2-core & $10,404$ & $24.53 \pm 3.30$ & $1.35$ & $205.11$ & $0.30$ & $206.76$ \\
DeepVGAE on 2-core & $10,404$ & $25.63 \pm 3.51$ & $1.35$ & $200.73$ & $0.30$ & $202.38$ \\
Graphite on 2-core & $10,404$ & $26.55 \pm 2.17$  & $1.35$ & $209.12$ & $0.30$ & $210.77$ \\
Var-Graphite on 2-core & $10,404$ & $\textbf{26.69} \pm \textbf{2.21}$ & $1.35$ & $200.86$ & $0.30$ & $202.51$ \\
ARGA on 2-core & $10,404$ & $23.68 \pm 3.18$  & $1.35$ & $207.50$ & $0.30$ & $209.15$ \\
ARVGA on 2-core & $10,404$ &$25.98 \pm 1.93$  & $1.35$ & $199.94$ & $0.30$ & $201.59$ \\
ChebGAE on 2-core & $10,404$ & $25.88 \pm 1.66$ & $1.35$ & $410.81$ & $0.30$ & $412.46$ \\
ChebVGAE on 2-core & $10,404$ &$26.50 \pm 1.49$  & $1.35$ & $399.96$ & $0.30$ & $401.61$ \\
\midrule
DeepWalk & - &$27.23 \pm 0.32$  & - & $342.25$ & - & $342.25$ \\
LINE & - & $26.26 \pm 0.28$  & - & $63.52$ & - & $63.52$ \\
Louvain & - & $23.02 \pm 0.47$ & - & $\textbf{27.32}$ & - & $\textbf{27.32}$ \\
node2vec & - & $\textbf{29.57} \pm \textbf{0.22}$ & - & $48.91$ & - & $48.91$ \\
\bottomrule
\end{tabular}
\caption{Node clustering on Pubmed ($n=19,717$, $m =44,338$), using VGAE on all cores, Graph AE/VAE variants on 2-core, and baselines.}
\end{footnotesize}
\end{table*}

%ok 

\begin{table*}
\centering
\begin{footnotesize}
\begin{tabular}{c|c|c|cccc}
\toprule
\textbf{Model}  & \textbf{Size of input} & \textbf{Mean Performance (in \%)} & \multicolumn{4}{c}{\textbf{Mean Running Times (in sec.)}}\\
& \textbf{$k$-core} & \tiny \textbf{MI} & \tiny $k$-core dec. & \tiny Model train & \tiny Propagation & \tiny \textbf{Total} \\ 
\midrule
VGAE on 14-core & $46,685$  & $\textbf{25.22} \pm \textbf{1.51}$ & $507.08$ & $6,390.37$ & $120.80$ & $7,018.25~(1\text{h}57)$ \\
on 15-core & $35,432$ & $24.53 \pm 1.62$ & $507.08$ & $2,589.95$ & $123.95$ & $3,220.98~(54\text{min})$  \\
on 16-core & $28,153$ & $24.16 \pm 1.96$ & $507.08$ & $1,569.78$ & $123.14$ & $2,200.00~(37\text{min})$ \\
on 17-core & $22,455$ & $24.14 \pm 2.01$ & $507.08$ & $898.27$ & $124.02$ & $1,529.37~(25\text{min})$ \\
on 18-core & $17,799$ & $22.54 \pm 1.98$ & $507.08$ & $\textbf{551.83}$ & $126.67$ & $\textbf{1,185.58~(20\text{min})}$ \\
\midrule
GAE on $15$-core &$35,432$ &$23.76 \pm 2.25$ & $507.08$ & $2,750.09$ & $123.95$ & $3,381.13~(56\text{min})$ \\
DeepGAE on $15$-core & $35,432$ & $24.27 \pm 1.10$ & $507.08$ & $3,007.31$ & $123.95$ & $3,638.34~(1\text{h}01)$ \\
DeepVGAE on $15$-core & $35,432$ & $24.54 \pm 1.23$ & $507.08$ & $2,844.16$ & $123.95$ & $3,475.19~(58\text{min})$ \\
Graphite on $15$-core & $35,432$ & $24.22 \pm 1.45$  & $507.08$ & $2,899.87$ &$123.95$ & $3,530.90~(59\text{min})$ \\
Var-Graphite on $15$-core & $35,432$ & $24.25 \pm 1.51$ & $507.08$ & $2,869.92$ & $123.95$ & $3,500.95~(58\text{min})$ \\
ARGA on $15$-core & $35,432$ & $24.26 \pm 1.18$  & $507.08$ & $3,013.28$ & $123.95$ & $3,644.31~(1\text{h}01)$ \\
ARVGA on $15$-core & $35,432$ &$24.76 \pm 1.32$  & $507.08$ & $2,862.54$ & $123.95$ & $3,493.57~(58\text{min})$ \\
ChebGAE on $15$-core & $35,432$ & $25.23 \pm 1.21$ & $507.08$ & $5,412.12$ & $123.95$ & $6,043.15~(1\text{h}41)$ \\
ChebVGAE on $15$-core & $35,432$ & $\textbf{25.30} \pm \textbf{1.22}$  & $507.08$ & $5,289.91$ & $123.95$ & $5,920.94~(1\text{h}38)$ \\
\midrule
LINE & - & $23.19 \pm 1.82$ & - & $33,063.80$ & - & $33,063.80~(9\text{h}11)$ \\
Louvain & - & $11.99 \pm 1.79$ & - & $13,634.16$ & - & $13,634.16~(3\text{h}47)$ \\
node2vec & - & $\textbf{24.10} \pm \textbf{1.64}$ & - & $26,126.01$ & - & $26,126.01~(7\text{h}15)$ \\
\bottomrule
\end{tabular}
\caption{Node clustering on Patent ($n = 2,745,762$, $m = 13,965,410$), using VGAE on 14 to 18 cores, graph AE/VAE variants on $15$-core, and baselines.}
\end{footnotesize}
\end{table*}
\end{document}